\documentclass[10pt]{article} 

\newif\ifsup
\supfalse\suptrue		

\usepackage[accepted]{aistats2019}
\usepackage{algorithm}
\usepackage{multirow}
\usepackage{booktabs}
\usepackage{pifont}				
\usepackage[utf8]{inputenc}		
\usepackage[T1]{fontenc}		
\usepackage{url}				
\usepackage{booktabs}       	
\usepackage{amsfonts}       	
\usepackage{nicefrac}       	
\usepackage{microtype}      	
\usepackage{amsmath,amsthm, dsfont}
\usepackage{amssymb,amsbsy,epsfig,float}
\usepackage{graphicx,wrapfig,lipsum}
\usepackage{algpseudocode}
\usepackage[makeroom]{cancel}
\usepackage{xspace}
\usepackage{mathtools}
\usepackage{bbm}
\usepackage{color}
\usepackage[usenames,dvipsnames]{xcolor}
\usepackage{enumerate, cases}
\usepackage{thmtools,thm-restate}
\usepackage[square, numbers]{natbib}	
\usepackage{caption}
\usepackage{subcaption}
\usepackage{tabularx}
\usepackage{pifont}
\allowdisplaybreaks
\usepackage[none]{hyphenat}
\usepackage{multicol}
\usepackage{changes}

\usepackage[
backgroundcolor = White,textwidth=\marginparwidth]{todonotes}

\usepackage{hyperref}			
\hypersetup{
	colorlinks   = true,		
	urlcolor     = blue, 		
	linkcolor    = blue, 		
	citecolor   = blue 			
}
\usepackage[capitalize]{cleveref}

\usepackage{etoolbox}
\makeatletter
\patchcmd\@combinedblfloats{\box\@outputbox}{\unvbox\@outputbox}{}{%
	\errmessage{\noexpand\@combinedblfloats could not be patched}%
}%
\makeatother

\newcommand{\iset}[1]{[#1]}

\newcommand{\SA}{\mathrm{SA}}
\newcommand{\SD}{\mathrm{SD}}
\newcommand{\WD}{\mathrm{WD}}

\newcommand{\TSA}{\Theta_{\SA}}
\newcommand{\Alg}{\mathcal{A}}
\newcommand{\TSD}{\Theta_{\SD}}
\newcommand{\TWD}{\Theta_{\WD}}

\newcommand{\EE}[1]{\mathbb{E}\left[#1\right]}

\newcommand{\Prob}[1]{\mathbb{P}\left\{#1\right\}}
\newcommand{\Regret}{\mathcal{R}}
\newcommand{\R}{\mathbb{R}} 
\newcommand{\Yti}{Y_t^i}
\newcommand{\Ytj}{Y_t^j}

\newcommand{\Ysi}{Y_s^i}
\newcommand{\Ysj}{Y_s^j}

\newcommand{\Yt}{Y_t}
\newcommand{\Yi}{Y^i}
\newcommand{\Yj}{Y^j}
\newcommand{\Yis}{Y^{i^\star}}

\newcommand{\pij}{p_{ij}}
\newcommand{\pijs}{p_{i^\star j}}
\newcommand{\hpij}{\hat{p}_{ij}(t)}
\newcommand{\hpji}{\hat{p}_{ji}(t)}
\newcommand{\hpijs}{\hat{p}_{i^\star j}(t)}
\newcommand{\hpjis}{\hat{p}_{ji^\star}(t)}
\newcommand{\hpijx}{\hat{p}_{i^\star j,s}}

\newcommand{\Dij}{\mathcal{D}_{ij}(t)}
\newcommand{\Nij}{\mathcal{N}_{ij}(t)}
\newcommand{\Dijm}{\mathcal{D}_{ij}(t-1)}
\newcommand{\Nijm}{\mathcal{N}_{ij}(t-1)}

\newcommand{\Nijms}{\mathcal{N}_{i^\star j}(t-1)}

\newcommand{\ceil}[1]{\left\lceil #1 \right\rceil}
\newcommand{\one}[1]{\mathds{1}_{\left\{#1\right\}}}
\newcommand\numberthis{\addtocounter{equation}{1}\tag{\theequation}}

\newcommand{\ind}[1]{\mathbb{I}\{#1\}}

\newtheorem{thm}{Theorem}
\newtheorem{lem}{Lemma}
\newtheorem{prop}{Proposition}
\newtheorem{cor}{Corollary}

\newtheorem{rem}{Remark}
\newtheorem{defi}{Definition}

\begin{document}

%

%

\twocolumn[
 	\aistatstitle{Online Algorithm for Unsupervised Sensor Selection}
	\aistatsauthor{
						Arun Verma 
                        \And 
                        Manjesh K. Hanawal 
                        \And  
                        Csaba Szepesv\'ari
                        \And 
                        Venkatesh Saligrama
	}

	\aistatsaddress{
                        Dept. of IEOR  \\ IIT Bombay, India \\v.arun@iitb.ac.in
                        \And 
                        Dept. of IEOR  \\ IIT Bombay, India \\mhanawal@iitb.ac.in 
                        \And 
                        DeepMind \\ London, UK \\ szepi@google.com
                        \And  
                        Dept. of ECE \\ Boston University, USA \\srv@bu.edu
	} 
]

\begin{abstract}
	In many security and healthcare systems, the detection and diagnosis systems use a sequence of sensors/tests. Each test outputs a prediction of the latent state and carries an inherent cost. However, the correctness of the predictions cannot be evaluated due to unavailability of the ground-truth annotations.
	Our objective is to learn strategies for selecting a test that gives the best trade-off between accuracy and costs in such \underline{u}nsupervised \underline{s}ensor \underline{s}election (USS) problems. Clearly, learning is feasible only if ground truth can be inferred (explicitly or implicitly) from the problem structure.  It is observed that this happens if the problem satisfies the `Weak Dominance' (WD) property.  We set up the USS problem as a stochastic partial monitoring problem and develop an algorithm with sub-linear regret under the WD property. We argue that our algorithm is optimal and evaluate its performance on problem instances generated from synthetic and real-world datasets.
\end{abstract}

\section{Introduction}
\label{sec:uss_introduction}

In many applications, one has to trade-off between accuracy and cost. For example, for detecting some event, it is not only the accuracy of a sensor that matters, but the associated sensing cost is important as well. Also, one may have to predict labels of instances for which ground-truth cannot be obtained. In such scenarios, feedback about the correctness of sensors' predictions remains unknown. Problems with this structure arise naturally in healthcare, security, and crowd-sourcing applications. In healthcare, the patients may not reveal the outcome of treatment due to privacy concerns; hence the effectiveness of the treatment is unknown. In crowd-sourcing systems, the expertise of self-listed-agents (workers) may not be known; therefore their quality cannot be identified. In a security application, specific threats may not have been seen before, and thus their in-situ ground-truth may not be available.  

In this work, we focus on the study of sensor selection problems where we do not have the advantage of knowing the ground-truth and hence cannot measure the error rates of the sensors. Here sensors could correspond to medical tests (healthcare),  detectors/scanners (security) or workers (crowd-sourcing).  In these  \underline{u}nsupervised \underline{s}ensor \underline{s}election (USS) problems, the goal is to still find the `best' sensor that gives the best trade-off between error and cost \citep{AISTATS17_hanawal2017unsupervised}. 

In  USS setup, it is assumed that the sensors form a cascade, i.e., they are ordered by their prediction efficiency and costs-- the average prediction error decreases hence, prediction efficiency increases with every stage of the cascade while the cost of acquiring it increases. Even though it is assumed that the sensor ordering is known and better sensors are associated with higher costs, the exact values of sensor errors are still unknown. The learner's goal is to find a sensor that has small value of total prediction cost for a given task, which includes both the cost of acquiring the sensor's outputs and the cost due to incorrect predictions. 

Clearly, without the knowledge of the ground-truth, one cannot find the optimal sensor as the sensor accuracies cannot be computed. In the USS setup, the structure of the problem is exploited, and it is shown that under certain conditions, namely strong dominance (SD) and weak dominance (WD), learning is possible. The SD property requires the prediction accuracy of a sensor to stochastically dominate prediction accuracy of other sensors with lower costs in the cascade. Specifically, it assumes that if a sensor's prediction is correct, then all the sensors that follow this sensor in the cascade also have correct predictions. 

Under the SD property, \citet{AISTATS17_hanawal2017unsupervised} established that USS problem is equivalent to a multi-armed bandit with side observations and exploit the equivalence to give an algorithm with sub-linear regret. SD property is quite strong and posits that disagreement probability of the predictions of two sensors is equal to the difference in error rates. This property implies that we can measure accuracy by measuring disagreement probabilities leading to a direct multi-armed bandit (MAB) reduction and analysis.

The WD property relaxes strict stochastic ordering on predictions and allows errors on some instances from better sensors. It is argued that the set of instances satisfying the WD property is maximally learnable, and any further relaxation of this property renders the problems unlearnable. The reduction techniques used under SD property does not apply/extend to WD property. For this case, a heuristic algorithm without any performance guarantee is given in \cite{AISTATS17_hanawal2017unsupervised}.  Our work bridges this gap. Our contributions are summarized as follows:


\begin{itemize}
	\item We develop an algorithm named \ref{alg:USS_WD} that has sublinear regret under WD property. We characterize regret in terms of how `well' the problem instances satisfy the WD property and then provide a bound that holds uniformly for all WD instances.
	
	\item 
	We give problem independent bounds on the regret of \ref{alg:USS_WD}. We show that it is of order $T^{2/3}$ under WD property and improves to $T^{1/2}$ under SD property. We establish that the bounds are optimal using results from partial monitoring in \cref{sec:uss_wd}.
	
	\item \citeauthor{AISTATS17_hanawal2017unsupervised} assume that sensors are ordered, i.e., their accuracy improve with their index, and used this fact in their algorithms. We relax this assumption in \cref{sec:uss_scd} where the sensors can have an arbitrary order. For this setup, we show that the same WD property determines the learnability. 
	
	\item We demonstrate performance of our algorithm on both synthetic and real datasets in \cref{sec:uss_experiments}. The experimental results show that regret of \ref{alg:USS_WD} is always lower than the heuristic algorithm in \cite{AISTATS17_hanawal2017unsupervised} (See Fig. \eqref{fig:compare_algs} in \cref{sec:uss_experiments}).

\end{itemize}

\subsection{Related Work} 
Several works consider the problem of sensor selection in either batch, or online settings (e.g., \citet{AISTATS13_trapeznikov2013supervised,ICML14_seldin2014prediction}). 
However, they all require that the label of each data point is available or the reward is obtained for each action. 
\citet{NIPS13_zolghadr2013online} considers that the labels are available on payment. \citet{AI02_greiner2002learning, ICML09_poczos2009learning} consider costs associated with tests. However, they assume that loss/reward associated with the players' action is revealed. In contrast, in our setting, the labels are not revealed at any point and are thus completely unsupervised, and the cost in our setup is related to sensing cost and not that of acquiring a label. 

\citet{UAI14_platanios2014estimating, ICML16_platanios2016estimating, NIPS17_platanios2017estimating} consider the problem of estimating accuracies of the multiple binary classifiers with unlabeled data. Most of these works make strong assumptions such as independence given the labels, knowledge of the true distribution of the labels.  
\citet{UAI14_platanios2014estimating} proposed logistic regression based methods using the classifiers' agreement rates over unlabeled data, \cite{ICML16_platanios2016estimating} extend this work to use graphical models, and \citet{NIPS17_platanios2017estimating} proposes method using probabilistic logic. Further, \citet{NIPS17_platanios2017estimating} also uses weighted majority vote for label prediction. All this is in the batch setting and differs from our online setup.

In the crowd-sourcing problems, various methods have been proposed to estimate unknown skill-level of crowd-workers from the noisy labels they provide (\citet{NIPS17_bonald2017minimax, ICLM18_kleindessner2018crowdsourcing}). These methods assume that all workers are having the same cost and aggregate the predictions on a  given dataset for estimating the accuracy of each worker. Unlike ours, these methods are not online.

Our work is closely related to the stochastic partial monitoring setting  
\citep{MOR06_cesa2006regret,ICML12_bartok2012partial,MOR14_bartok2014partial,NIPS15_wu2015online}, where the feedback from actions is indirectly tied to the rewards.  In our setting, we exploit the problem structure to learn an optimal arm without explicitly knowing the loss associated with each action.

\section{USS Problem}
\label{sec:uss_setup}

We cast the {unsupervised, stochastic, cascaded sensor selection} as an instance of stochastic partial monitoring problem (SPM). We use sensor and arm interchangeably in the following. Formally, a problem instance in our setting is specified by a pair $\theta = (P,c)$, where $P$ is a distribution over the $K+1$ dimensional hypercube, and $c$ is a $K$-dimensional, non-negative valued vector of costs. While $c$ is known to the learner from the start, $P$ is unknown. Henceforth, we identify problem instance by $\theta$.
\noindent
The instance parameters specify the learner-environment interaction as follows: In each round $t=1,2,\dots$, the environment generates a $K+1$-dimensional binary vector 
$(Y_t,Y_t^1,\dots,Y_t^K) \in \{0,1\}^{K+1}$ chosen at random from $P$. Here, $Y_t^j$ is the output of sensor $j$, while $Y_t$ is the (hidden) label to be guessed by the learner. Simultaneously, the learner chooses an index $I_t\in [K]$ where $[K] = \{1,2,\ldots, K\}$, and observes the sensor outputs $Y_t^1,\dots,Y_t^{I_t}$, i.e., the learner goes through the first $I_t$ sensors and observes their outputs. Dropping the subindex $t$, write $S=(Y^1,\dots,Y^K) \in \{0,1\}^K$. Then, $P$, the joint probability distribution of $Y$ and $S$, can be expressed as $P=P_S\otimes P_{Y|S}$, where for any $s \in \{0,1\}^K$ and $y \in \{0,1\}$, $P_S(s)=\Prob{S=s}$ is (essentially) observable while $P_{Y|S}(y|s)=\Prob{Y=y|S=s}$ is not.

\begin{figure}[H]
	\centering
	\includegraphics[width=\linewidth]{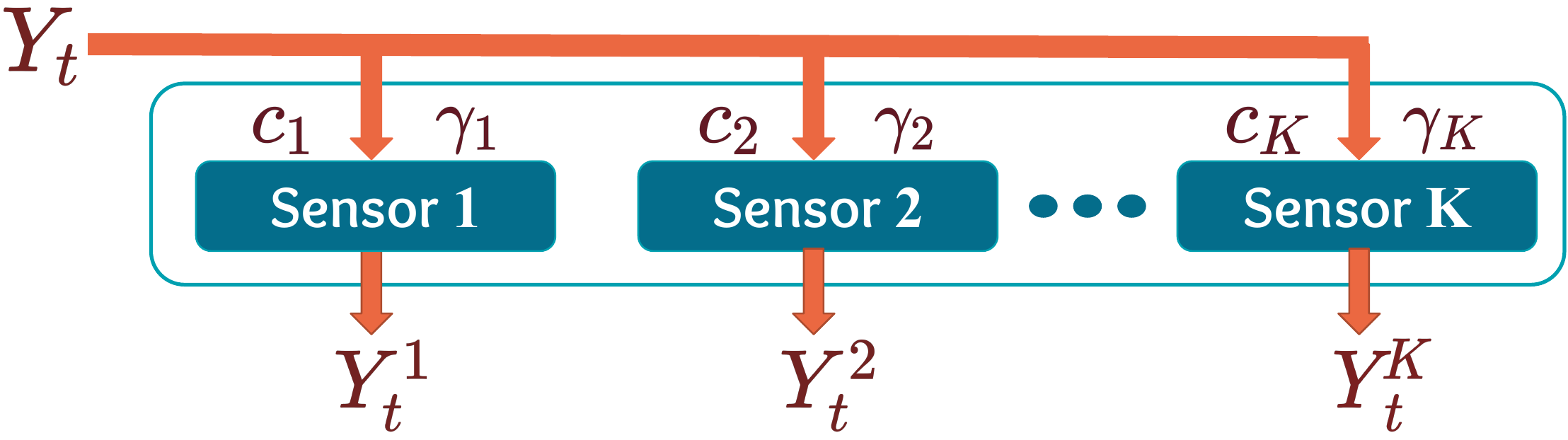}
	\caption{\footnotesize Cascaded Unsupervised Sensor Selection setup. $Y_t$ is the hidden state of the instance and $Y_t^1, Y_t^2 \ldots Y_t^K$ are sensor outputs. $c_j$ denotes the cost of using the sensor $j$ and $\gamma_j$ denotes error rate of the sensor $j$.}
	\label{fig:SensorCascade}
	\vspace{-5mm}
\end{figure}   

\citeauthor{AISTATS17_hanawal2017unsupervised} in addition assumes
that the sensors are known to be ordered from least accurate to most accurate, i.e., $\gamma_j \doteq\gamma_j(\theta) \doteq \Prob{Y \ne \Yj}$ is decreasing in $j$. We relax this assumption later in the \cref{sec:uss_scd}. The cost associated with sensor $j\in \iset{K}$ is denoted by $c_j\geq 0$ and the cost of choosing action $I_t$  is $C_{I_t} \doteq c_1 + \dots + c_{I_t}$, as the selection has to be done sequentially. The {\it total cost} incurred by the learner in round $t$ is thus $\lambda C_{I_t} + \ind{Y_t \ne Y_t^{I_t}}$ where $\lambda$  is a trade-off parameter between error rate and cost of using sensors\footnote{$\lambda$ is a parameter that makes associated cost unit-less. For example, assume cost is in \$ and associated $\lambda$ is $p$. If cost is increased by multiple of $s$ ($s=100$ for cost in cents) then the corresponding $\lambda$ will be $p/s$ and vice-versa.}. Without loss of generality, we set $\lambda= 1$. The goal of the learner is to compete with the best choice knowing the $\theta$. Let $c(j,\theta) \doteq \EE{ C_{j} +\ind{\Yt \ne \Ytj }}  (= C_j +\gamma_j)$ and $i^\star \doteq i^\star(\theta) \doteq$ $\max\{l: l= \arg\min\limits_{j \in [K]}c(j, \theta)\}$ be the optimal sensor. The cumulative (pseudo-)regret of the learner running an algorithm $\Alg$ up to the end of round $T$ is
\begin{align}
	\label{equ:regret}
	\Regret_T(\Alg,\theta) &= \sum_{t=1}^T c(I_t,\theta) - Tc(i^\star,\theta).
\end{align}
We say that the (expected) regret is sublinear if $\mathbb{E}[\Regret_T]/T \to 0$ as $T\to \infty$, where the expectation is over $I_t$, which is random as it depends on past random data. When the regret is sublinear, the learner collects almost as much reward in expectation in the long run as an oracle that knew the optimal action from the beginning. 
Let $\TSA$ be the set of all stochastic, cascaded sensor selection problems.  Thus, $\theta \in \TSA$ such that $Y\sim \theta$ and $\gamma_j(\theta):=\Prob{Y\ne Y^k}$  is decreasing in $k$. Given a subset $\Theta\subset \TSA$, we say that $\Theta$ is \emph{learnable}  if there exists a learning algorithm $\Alg$ such that for any $\theta\in \Theta$, the expected regret $\EE{ \Regret_T(\Alg,\theta) }$ of algorithm $\Alg$ on instance $\theta$ is sub-linear. 
A subset $\Theta$ is said to be a \emph{maximal learnable problem class} if it is learnable and for any subset $\Theta^\prime \subset \TSA$ that contains $\Theta$ is not learnable.

\subsection{Strong and Weak Dominance}
\label{ssec:SD_WD}
The purpose of this section is to introduce the notions of strong and weak dominance from the work of \citet{AISTATS17_hanawal2017unsupervised}. While \citeauthor{AISTATS17_hanawal2017unsupervised} studied learning under strong dominance, here we will focus on weak dominance. 
We also modify the definition of weak dominance of \citeauthor{AISTATS17_hanawal2017unsupervised} to correct an oversight of them.

The  strong dominance (SD) property is defined as follows:
\begin{defi}[Strong Dominance (SD)]
	\label{def:SD}
	An instance $\theta \in \Theta_{SA}$ is said to satisfy the strong dominance property if for $(Y, Y^1, \cdots , Y^K) \sim \theta $, it holds almost surely (a.s.) that
	\begin{align} 
		\label{equ:sd_prop}
		\Yi = Y \mbox{ for some } i \in[K] \Rightarrow  Y^{j} = Y \;\;\forall j > i.
	\end{align}
\end{defi}
The $\SD$ property implies that if a sensor predicts correctly then, a.s., all the sensors in the subsequent stages of the cascade also predict correctly. The set of all instances satisfying $\SD$ property, i.e., $\TSD = \left\{\theta \in \TSA:\theta \mbox{ satisfies $\SD$ condition}\right\}$ is learnable \cite[Theorem 2]{AISTATS17_hanawal2017unsupervised}. The weaker version of the $\SD$ property is defined as follows:
\begin{defi}[Weak Dominance (WD)] 
	\label{def:WD} An instance $\theta \in \TSA$  is said to satisfy \emph{weak dominance property} if
	\begin{align}
	    \rho(\theta) := \min_{j > i^\star} \frac{C_j-C_{i^\star}}{\Prob{Y^j \neq Y^{i^\star}}}>1.
	\end{align}
\end{defi}
Let $\TWD = \left\{\theta \in \TSA:\theta \mbox{ satisfies $\WD$ condition}\right\}$ denote the set of instances satisfying the $\WD$ property.
The WD property holds for all problem instances where sensor $K$ is an optimal sensor.

\citet{AISTATS17_hanawal2017unsupervised} claimed that $\TWD$ is learnable. However, their definition allowed $\rho(\theta)\ge 1$. As it turns out, permitting $\rho(\theta)=1$ can prevent $\TWD$ from being learnable:

\begin{prop}\label{prop:twdbug}
	The set $\TWD'=\{ \theta\in \TSA\,:\, \rho(\theta)\ge 1 \}$ is not learnable.
\end{prop}

\begin{proof}
	Let $C_2-C_1=1/4$.
	Theorem~19 of  \citet{AISTATS17_hanawal2017unsupervised} constructs instances $\theta,\theta'\in \TWD'$ such 
	that the optimal decision for $\theta$ is sensor $1$, for $\theta'$ is sensor $2$.
	The suboptimality gap on instance $\theta$ is $1/4$, while on instance $\theta'$ is $\epsilon$, where $\epsilon\in [0,1]$ is a tunable parameter. At the same time $\Prob{Y^1\ne Y^2}=1/4$ in $\theta$ and $\Prob{Y^1\ne Y^2}=1/4+\epsilon$  in $\theta'$. Theorem~17 of  \citet{AISTATS17_hanawal2017unsupervised} implies that a sound algorithm must check $1/4=C_2-C_1 \ge \Prob{Y^1\ne Y^2}$. However, no finite amount of data is sufficient to decide this: In particular, one can show that if an algorithm on $\theta$ achieves sublinear regret, then it must suffer linear regret on $\theta'$ for $\epsilon>0$ small enough. Hence, all algorithms will suffer linear regret on some instance in $\TWD'$. 
\end{proof}

The following theorem is obtained directly from Theorem $14$ and Theorem $19$ in \cite{AISTATS17_hanawal2017unsupervised} after excluding the case $\rho=1$ in their proofs.  
\begin{thm}
	The set $\TWD$ is a maximal learnable set.
\end{thm}
In the following, we use an alternate characterization of the $\TWD$ property given as
\begin{equation}
\xi := \min_{j>i^\star} \Big\{C_j - C_{i^\star} - \Prob{\Yis \ne \Yj} \Big\} > 0
\end{equation}
Notice that $\rho>1$ if and only if $\xi>0$. Larger the value of $\xi$ `stronger' is the $\TWD$ property and easier it is to identify an optimal action. We later characterize the regret bounds in terms of $\xi$.

\section{Algorithm Under WD Property}
\label{sec:uss_wd}

In the following, we let $i^\star$ denote the optimal arm with largest index, i.e., $i^\star= \max\{l: l= \arg\min\limits_{j \in [K]}c(j, \theta)\}$. 
The optimal sensor $i^\star$ satisfies the following inequalities:
\vspace{-4.5mm}
\begin{subequations}
	\label{eq:cost_exp_err}
	\begin{align}
		&\forall j<i^\star \,:\, C_{i^\star} - C_j \leq \gamma_j-\gamma_{i^\star} \,, \label{eq:wd1}\\ 
		&\forall j>i^\star \,:\, C_j - C_{i^\star} > \gamma_{i^\star}-\gamma_j \,. \label{eq:wd2}
	\end{align}
\end{subequations}
Note that the above decision criteria is risk-averse, i.e., if two sensors have the same optimal cost, the sensor with smaller error-rate will be chosen.

A natural candidate for a decision criteria is to replace error rates ($\gamma_j$) by their estimates and look for an index that satisfies (\ref{eq:wd1}) and (\ref{eq:wd2}).
However, error rates ($\gamma_j$) cannot be estimated,  implying that \eqref{eq:wd1} and  \eqref{eq:wd2} can not lead to a sound algorithm. 
Recall the following result from \cite{AISTATS17_hanawal2017unsupervised}:
\begin{prop}[{\cite[Proposition 3]{AISTATS17_hanawal2017unsupervised}}]
	\label{prop:error_prob}
	Let $\gamma_i = \gamma_i(\theta)$ for any $\theta$, not necessarily in $\TSA$. Then, for any $i,j \in [K]$, $	\gamma_i - \gamma_j = \Prob{Y^i \neq Y^j} - 2\Prob{Y^i = Y, Y^j  \neq Y},$ and hence $\gamma_i - \gamma_j \leq \Prob{Y^i \neq Y^j}$.
\end{prop}

Using Proposition \ref{prop:error_prob}, criteria \eqref{eq:wd1} implies 
\begin{align}
	\forall j<i^\star \,:\, C_{i^\star} - C_j \leq  \Prob{\Yis \ne \Yj} \label{eq:prob_wd1}
\end{align}
where $\Prob{\Yis \ne \Yj}$ forms a proxy for $\gamma_j-\gamma_{i^\star}$. 
For the case $j > i^\star$, we can appeal to the $\WD$ property and  can replace \eqref{eq:wd2} by
\vspace{-1.25mm}
\begin{equation}
	\forall j>i^\star \,:\, C_j - C_{i^\star} > \Prob{\Yis \ne \Yj} \label{eq:prob_wd2}
\end{equation}
$\Prob{\Yis \ne \Yj}$ can be estimated as the distribution $P_S$ is observable. Motivated by \eqref{eq:prob_wd1} and \eqref{eq:prob_wd2}, we define the selection criteria based on the following sets:
\vspace{-1.25mm}
\begin{align}
	&\hspace{-.17cm}\mathcal{B}^l= \Big\{i: \forall j<i, C_i - C_j \leq \Prob{\Yi \ne \Yj}\Big\} \cup \{1\} \label{set:Bl} \\
	&\hspace{-.17cm}\mathcal{B}^h\hspace{-0.5mm} = \hspace{-0.5mm}\Big\{\hspace{-0.5mm}i: \forall j>i, C_j - C_i > \Prob{\Yi \ne \Yj} \hspace{-0.6mm}\Big\}\hspace{-0.5mm} \cup \hspace{-0.5mm}\{K\} \label{set:Bh}.
\end{align}
\vspace{-3mm}
\begin{restatable}{lem}{SetB}
	\label{lem:B}
	Let $\theta \in \TWD$. Let $\mathcal{B} := \mathcal{B}(\theta)=\mathcal{B}^l \cap \mathcal{B}^h$. Then 
	$\mathcal{B}$ contains the optimal sensor.  
\end{restatable}
\vspace{-1mm}
\ifsup
The proof is in \cref{sec:lemsetbproof}.
\else
The proof can be found in the supplementary material.
\fi

\subsection{\ref{alg:USS_WD}}
In bandit problems, the upper confidence bound (UCB) \citep{ML02_auer2002finite,COLT11_garivier2011kl} is highly effective for dealing with the trade-off between exploration and exploitation. Using UCB idea, we develop an algorithm, named \ref{alg:USS_WD}, that utilizes the sets \eqref{set:Bl} and \eqref{set:Bh} and looks for an index that belongs to both.  Since disagreement probabilities, $\left(\pij \doteq \Prob{\Yi \ne \Yj}\right)$'s, are unknown (but fixed), they are replaced by their optimistic empirical estimates at round $t$, denoted by $\hpij + \Psi_{ij}(t)$ where $\hpij$ is empirical estimate of $\pij$ and $\Psi_{ij}(t)$ is the confidence term associated with $\hpij$ as in UCB algorithm. The new sets for selection criteria are defined as follows:
\vspace{-1.25mm}
\begin{subequations}
	\label{set:eoptimal}
	\begin{align}
		&\hspace{-1.25mm}\mathcal{\hat{B}}_t^l\hspace{-0.5mm}=\hspace{-0.1cm} \{i: \forall j<i, C_{i}-C_{j} \leq \hpji \hspace{-0.5mm}+\hspace{-0.5mm} \Psi_{ji}(t)\} \hspace{-0.25mm}\cup\hspace{-0.25mm} \{1\},\hspace{-0.5mm} \label{set:eBl} \\
		&\hspace{-1.25mm}\mathcal{\hat{B}}_t^h\hspace{-0.5mm}=\hspace{-0.1cm}\{i: \forall j\hspace{-0.5mm}>\hspace{-0.5mm}i;C_{j} \hspace{-0.5mm}-\hspace{-0.5mm}C_{i} \hspace{-0.5mm}>\hspace{-0.5mm} \hpij \hspace{-0.5mm}+\hspace{-0.5mm} \Psi_{ij}(t)\} \hspace{-0.25mm}\cup\hspace{-0.25mm} \{K\}. \hspace{-0.5mm}\label{set:eBh}
	\end{align}
\end{subequations}
From the definition, it is easy to verify that $\hpij = \hpji$ and $\Psi_{ij}(t) = \Psi_{ji}(t)$ for any $(i,j)$ pair. Therefore, it is enough for algorithm to only keep track of $\hpij$ and $\Psi_{ji}(t)$ for $i<j$. 
\begin{rem}
	\label{rem:DecisionSet}
	It might be tempting to use lower confidence, i.e., $\hpij - \Psi_{ij}(t)$ term instead of the upper confidence term in \eqref{set:eBh}.  However, such a change can make the algorithm converge to a sub-optimal sensor. A detailed discussion is given in the supplementary material.
\end{rem}

\begin{center}
	\begin{algorithm}[!t]
		\renewcommand{\thealgorithm}{USS-UCB}
		\floatname{algorithm}{}
		\caption{Algorithm for \textbf{USS} under \textbf{WD} property} 
		\label{alg:USS_WD}
		\begin{algorithmic}[1]
			\Statex \hspace{-0.44cm}\textbf{Input:} $\alpha>0.5$
			\State Select sensor $I_1 = K$ and observe $Y^1_1,\dots,Y^{I_1}_1$
			\State Set $\mathcal{D}_{ij}(1) \leftarrow \one{\Yi_1 \ne \Yj_1}, \;\mathcal{N}_{ij}(1) \leftarrow 1\;\; \forall i< j \le I_1$\label{alg:USS_WD_ini}
			\For{$t=2,3,...$}
				\State $\hpij \leftarrow \frac{\Dijm}{\Nijm}\;\; \forall i< j \le K$\label{alg:USS_WD_est}
				\State $\Psi_{ij}(t) \leftarrow \sqrt{\frac{\alpha\log f(t)}{\Nijm}}\;\; \forall i< j \le K$\label{alg:USS_WD_conf} 
				\State Compute $\mathcal{\hat{B}}_t^l$ and $\mathcal{\hat{B}}_t^h$ as given in \eqref{set:eBl} and \eqref{set:eBh}\label{alg:USS_WD_sets}
				\State $\mathcal{\hat{B}}_t := \mathcal{\hat{B}}_t^l \cap \mathcal{\hat{B}}_t^h$
				\State $I_t\leftarrow \min \big\{\mathcal{\hat{B}}_t\cup \{K\}\big\}$\label{alg:USS_WD_select}
				\State Select sensor $I_t$ and observe $Y^1_t,\dots,Y^{I_t}_t$\label{alg:USS_WD_Ys}
				\State $\Dij \leftarrow \Dijm + \one{\Yti \ne \Ytj}\;\; \forall i< j \le I_t$\label{alg:USS_WD_D} 
				\State $\Nij \leftarrow \Nijm + 1\;\; \forall i< j \le I_t$\label{alg:USS_WD_N}
			\EndFor
		\end{algorithmic}
	\end{algorithm}
\end{center}
\vspace{-2mm}
The pseudo code of USS-UCB is given in Algorithm \ref{alg:USS_WD} and it works as follows. It takes $\alpha$ as an input that trades-off between exploration and exploitation. In the first round, it selects sensor $K$ and  initializes the value of number of comparisons and counter of disagreements for each pair $(i,j), i<j$, denoted $\mathcal{N}_{ij}(1)$ and $\mathcal{D}_{ij}(1)$ (Line \ref{alg:USS_WD_ini}), respectively. In each subsequent round, the algorithm computes estimate for the disagreement probability $\hpij$ (Line \ref{alg:USS_WD_est})and the associated confidence $\Psi_{ij}(t)$ (Line \ref{alg:USS_WD_conf})). Then $\hpij$ and $\Psi_{ij}(t)$ are used for computing sets $\mathcal{B}_t^l$ and $\mathcal{B}_t^h$ (Line \ref{alg:USS_WD_sets}) which are then used to select the sensor. Specifically, the algorithm selects a sensor $I_t$ that satisfies \eqref{set:eBl} and \eqref{set:eBh} (Line \ref{alg:USS_WD_select}). 

Since initial estimates for $\pij$ are not good enough, $\hat{ \mathcal{B}}_t$ can be empty. In such a case, the algorithm selects the sensor $K$.  
After selection of sensor $I_t$, $Y^j_t, j\in [I_t] $ (Line \ref{alg:USS_WD_Ys}) are observed which are then used to update the $\Dij$ (Line \ref{alg:USS_WD_D}) and $\Nij$ (Line \ref{alg:USS_WD_N}) in the algorithm.

\subsection{Regret Analysis}
Following notations and definition are useful in subsequent proofs. For the optimal sensor $i^\star$ and each $j \in [K]$, let 
\begin{align}
	\Delta_j := C_j + \gamma_j - (C_{i^\star} + \gamma_{i^\star}) \label{def_delta},
\end{align}
\begin{subnumcases}
	{\kappa_j := }
	\pijs - (\gamma_j - \gamma_{i^\star}),\;\; \text{ if } j<i^\star  \label{def_kappa_l}\\
	\pijs - (\gamma_{i^\star} - \gamma_j), \;\; \text{ if } j>i^\star \label{def_kappa_h} 
\end{subnumcases}
\begin{subnumcases}
	{\xi_j := }
	\Delta_j + \kappa_j, \;\; \text{ if } j<i^\star  \label{def_xi_l}\\
	\Delta_j - \kappa_j, \;\; \text{ if } j>i^\star \label{def_xi_h} 
\end{subnumcases}

Notice that the values of $\kappa_j$ and $\xi_j$ for all $j\in \iset{K}$ are positive under the $\WD$ property. 
\noindent
Let $N_j(T)$ denote the number of times sensor $j$ is selected until round $T$. The following proposition gives the mean number of times a sub-optimal sensor is selected. 

\begin{restatable}{prop}{MeanPulls}
	\label{prop:meanpulls}
	Let  $f(t)$ be a  positive valued increasing function such that $C=\lim\limits_{T\rightarrow\infty}\sum\limits_{t=1}^T\dfrac{1}{f(t)^{2\alpha}}< \infty$ in \textnormal{\ref{alg:USS_WD}}. For any $\theta \in \TWD$, the mean number of times a sensor $j \neq  i^\star$ is selected, is bounded as follows: 
	\begin{itemize}
		\item for any $j < i^\star$
		\begin{equation*}
			\EE{N_j(T)}  \le \dfrac{C}{2\xi_j^2},
		\end{equation*} 
		\item and for any $j>i^\star$
		\vspace{-1mm}
		\begin{equation*}
			\hspace{-.3cm}\EE{N_j(T)}\le\hspace{-0.05cm} 1 \hspace{-0.05cm}+\hspace{-0.05cm} \frac{1}{\xi_j^2}\hspace{-0.05cm}\left(\hspace{-0.05cm}\alpha\log f(T) \hspace{-0.05cm}+\hspace{-0.05cm} \sqrt{\frac{\pi\alpha\log f(T)}{2}} \hspace{-0.05cm}+\hspace{-0.05cm} \frac{1}{2}\hspace{-0.05cm}\right).
		\end{equation*} 
	\end{itemize}
\end{restatable}
Notice that the mean number of times a sensor $j < i^\star$ is selected, is finite. The regret bounds follows by noting that $\EE{\Regret_T}=\sum_{j < i^\star}\EE{N_j(T)}\Delta_j + \sum_{j > i^\star}\EE{N_j(T)}$ $\Delta_j$. Formally, we have the following regret bound. 
\begin{thm}
	\label{thm:regret}
	Let $f(t)$ be set as in \cref{prop:meanpulls}. Then, for any $\theta \in \TWD$, the expected regret of \textnormal{\ref{alg:USS_WD}} in $T$ rounds is bounded as below:
	\begin{align*}
		\EE{\Regret_T} \le \sum\limits_{j < i^\star} \dfrac{\Delta_j C}{2\xi_j^2} + &\sum\limits_{j > i^\star} \Delta_j \Bigg[1 \;+ \frac{1}{\xi_j^2}\Bigg(\alpha\log f(T) \\
		& \;+ \sqrt{\frac{\pi\alpha\log f(T)}{2}} + \frac{1}{2}\Bigg) \Bigg].
	\end{align*}
\end{thm}


\begin{cor}
	\label{cor:fix_para}
	Let $\alpha=1$ and $f(t) = t$ in \cref{thm:regret}. Then, expected regret of \textnormal{\ref{alg:USS_WD}} for any $\theta \in \TWD$ in $T$ rounds is of $O\left(\sum\limits_{j>i^\star}\frac{\Delta_j\log T}{\xi^2}\right)$.
\end{cor}

\begin{cor}
	\label{cor:sd}
	Let technical conditions stated in \cref{cor:fix_para} hold. Then expected regret of \textnormal{\ref{alg:USS_WD}} for any $\theta \in \TSD$  in $T$ rounds is of $O\left(\sum\limits_{j>i^\star}\frac{\log T}{\xi}\right)$.
\end{cor}
\vspace{-3mm}
\begin{proof}
	Since $|\gamma_j - \gamma_{i^\star}| = \pijs$ for $\theta \in \TSD$, $\kappa_j = 0, \forall j \in [K]\Rightarrow \xi_j = \Delta_j$. Rest follows from Corollary \ref{cor:fix_para}. 
\end{proof}

We next present problem independent bounds on the expected regret of \ref{alg:USS_WD}.
\begin{restatable}{thm}{probIndBound}
	\label{thm:prob_independent_bound}
	Let $f(t)$ be set as in \cref{prop:meanpulls}. The expected regret of \textnormal{\ref{alg:USS_WD}} in $T$ rounds 
	\vspace{-3mm}
	\begin{itemize}
		\item 	for any instance in $\TWD$ is bounded as
		\vspace{-3mm}
		\begin{align*}
			\EE{\Regret_T} \le 3\left(3\alpha K\log f(T)\right)^{1/3}T^{2/3}.
		\end{align*}
		\vspace{-8mm}
		\item for any instance in $\TSD$ is bounded as
		\vspace{-3mm}
		\begin{align*}
			\EE{\Regret_T} \le 4\left(\alpha KT\log f(T)\right)^{1/2}.
		\end{align*}
	\end{itemize}
\end{restatable}

\begin{cor}
The expected regret of \textnormal{\ref{alg:USS_WD}} on $\TSD$ is $\tilde{O}(T^{1/2})$ and on $\TWD$ it is $\tilde{O}(T^{2/3})$, where $\tilde{O}$ hides logarithmic terms.
\end{cor}
The proof of Theorem \ref{thm:prob_independent_bound} can be found in the supplementary material. We note that the above uniform bounds do not contradict Theorem $19$ in \cite{AISTATS17_hanawal2017unsupervised} which claimed non-existence of uniform bounds. The $\TWD$ condition considered in \cite{AISTATS17_hanawal2017unsupervised} incorrectly includes the class of instances satisfying $\rho=1$ which renders $\TWD$ not learnable, whereas in our definition of  $\TWD$ these instances are excluded and $\TWD$ is learnable.

{\bf Discussion on optimality of \ref{alg:USS_WD}:} Any partial monitoring problem can be classified as an `easy', `hard' or `hopeless' problem if it has  expected regret bounds of the order $\Theta(T^{1/2}), \Theta(T^{2/3})$ or $\Theta(T)$, respectively, and there exists no other class in between \cite{MOR14_bartok2014partial}. The class $\TSD$ is {regret equivalent} to a stochastic multi-armed bandit with side observations \cite{AISTATS17_hanawal2017unsupervised},  for which regret scales as $\Theta(T^{1/2})$, hence $\TSD$ resides in the easy class and our bound on it is optimal.  Since $\TWD \supsetneq \TSD$, $\TWD$ is not easy, and also $\TWD$ is learnable, it cannot be hopeless. Therefore, the class $\TWD$ is hard. We thus conclude that the regret bound of \ref{alg:USS_WD} is optimal in $T$. However, optimality concerning other leading constants (in terms of $K$) is to be explored further.

\section{Unknown Ordering of Sensors}
\label{sec:uss_scd}

The sensor error rates are unknown in our setup and cannot be estimated due to unavailability of ground-truth. Thus, it may happen that we do not know whether error rate of the sensors in the cascade is decreasing or not. In this section, we remove the requirement that sensors are arranged in the decreasing order of their error rates and allow them to be arranged in an arbitrary order that is unknown. We denote the set of USS instances with unknown ordering of sensors by their error-rates as $\TSA^\prime \supset\Theta_{SA}$. The rest of the setup is same as in Section \ref{sec:uss_setup}. We show that even with this relaxation, WD property defined earlier continues to characterize the learnability of the problem. 

We begin with the following observation.


\begin{lem}
	\label{lem:ErrorRateOrder}
	Let $i^\star$ be an optimal sensor. Then, error rate of any sensor $j<i^\star$ is higher than that of $i^\star$.
\end{lem}
\begin{proof}	
We have $\gamma_j -\gamma_{i^\star} \geq C_{i^\star}- C_i$ for all $j \in [K]$. For $j < i^\star$, $C_{i^\star}- C_j \geq 0$ as costs are increasing with sensors. Hence  $\gamma_j\geq\gamma_{i^\star}$.
\end{proof}
 The following corollary directly follows from Prop. \ref{prop:error_prob}.

\begin{cor}
For any $i,j \in [K]$, $\max\{0, \gamma_{j}-\gamma_{i}\}\leq \Prob{Y^i\neq Y^j}$.
\end{cor}

The following two propositions provide the conditions on sensor costs that allows comparison of their total costs based on disagreement probabilities.

\begin{restatable}{prop}{PropCostRangeHigh}
	\label{prop:CostRange1}
Let $i<j$. Assume 
\begin{equation}
	\label{eqn:CostRange1}
C_j -C_i \notin \left (\max\{0, \gamma_i-\gamma_j\}, \Prob{Y^i \neq Y^j}\right].
\end{equation}
Then, $C_j-C_i >  \max \{0, \gamma_i-\gamma_j \}$ iff $C_j-C_i > \Prob{Y^j \neq Y^i}$.
\end{restatable}

\begin{restatable}{prop}{PropCostRangeLow}
		\label{prop:CostRange2}
	Let $i>j$. Assume 
	\begin{equation}
			\label{eqn:CostRange2}
	C_i -C_j \notin \left (\max\{0, \gamma_j-\gamma_i\}, \Prob{Y^i \neq Y^j}\right ].
	\end{equation}
	Then, $C_i-C_j \leq  \max\{0, \gamma_j-\gamma_i\} $ iff $C_j-C_i \leq \Prob{Y^i \neq Y^j}$.
\end{restatable}

From Lemma (\ref{lem:ErrorRateOrder}), for any $j < i^*$ we have $\max\{0, \gamma_j -\gamma_{i^\star}\}=\gamma_j -\gamma_{i^\star}$. Propositions (\ref{prop:CostRange1}) and (\ref{prop:CostRange2}) then suggests that the value of $\Prob{Y^i \neq Y^j}$ are sufficient to select the optimal sensor if the sensors costs satisfy (\cref{eqn:CostRange1}) for all $j > i^\star$ and Eqn.~(\ref{eqn:CostRange2}) for all $j < i^\star$. Since the values of $\Prob{Y^i \neq Y^j}$ can be estimated for all $i,j \in [K],$ we can establish the following result.

\begin{restatable}{prop}{PropWDUnorder}
	\label{prop:WD2} 
	Let  $i^\star$ be an optimal sensor. Any problem instance $\theta \in \TSA^\prime$  is learnable if
	\begin{equation*}
		\forall\; j > i^\star \;\; C_j -C_{i^\star} \notin \left (\max\{0, \gamma_{i^\star}-\gamma_j\}, \Prob{Y^i \neq Y^j}\right ].
	\end{equation*}
\end{restatable}

Notice that  for $j > i^\star$, $C_j- C_i^\star\geq 0$ and $C_j-C_{i^\star}\geq \gamma_{i^\star}-\gamma_j $. Hence, the learnability condition reduces to $\forall\; j >i^\star, C_j-C_{i^\star} > \Pr\{Y^{i^\star}\neq Y^j\}$, i.e., same as the WD condition. Hence, we have the following result.

\begin{thm}
	The set $\{\theta \in \TSA^\prime: \rho(\theta)>1\} $ is learnable.
\end{thm}


\section{Experiments}
\label{sec:uss_experiments}
In this section, we evaluate the performance of \ref{alg:USS_WD} on different problem instances derived from synthetic and two `real' datasets: PIMA Indians Diabetes \citep{UCI16_pima2016kaggale} and Heart Disease (Cleveland) \citep{HEART98_robert1988va, UCI17_Dua2017}. In our experiments, each sensor is represented by a classifier that is arranged in order of their decreasing misclassification error, i.e., error-rate for each dataset. The cost of using a classifier is assigned based on its error-rate -- smaller the error-rate higher the cost. The case where sensors' error-rate need not to decrease in the cascade is also considered.


\begin{figure*}[!ht]
	\centering
	\begin{subfigure}[b]{0.325\textwidth}
		\includegraphics[width=\linewidth]{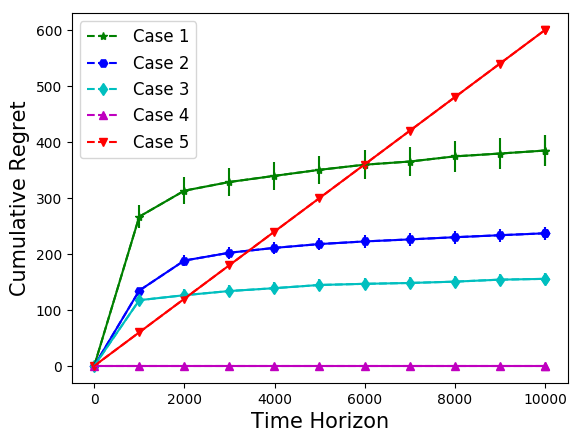}
		\caption{Synthetic BSC Dataset}
		\label{fig:regret_bsc}
	\end{subfigure}
	\begin{subfigure}[b]{0.325\textwidth}
		\includegraphics[width=\linewidth]{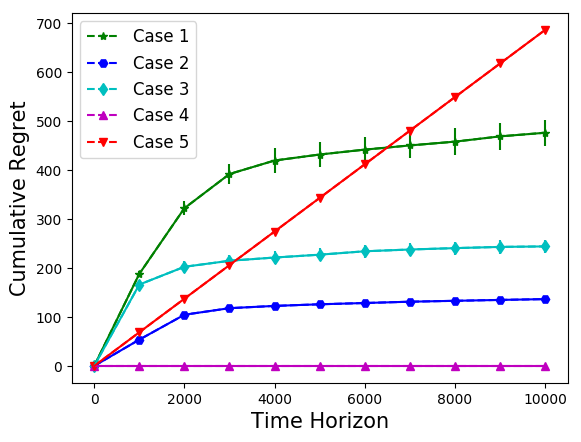}
		\caption{Pima Indian Diabetes}
		\label{fig:regret_diabetes}
	\end{subfigure}
	\begin{subfigure}[b]{0.325\textwidth}
		\includegraphics[width=\linewidth]{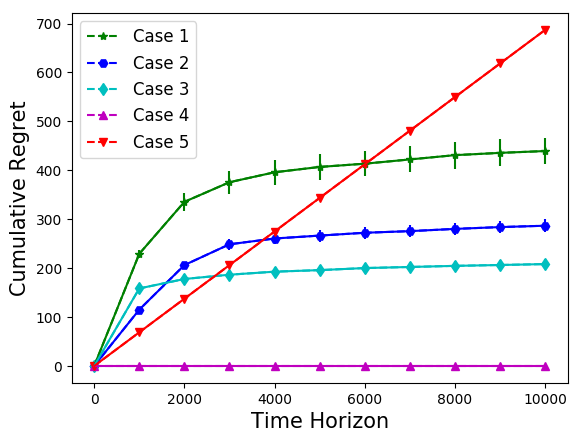}
		\caption{Heart Disease}
		\label{fig:regret_heart}
	\end{subfigure}
	\caption{\footnotesize Cumulative regret for different problem instances of \ref{alg:USS_WD} with parameter  $\alpha=0.51$.}
	\label{fig:regret}
\end{figure*}
\begin{figure*}[!ht]
	\centering
	\begin{subfigure}[b]{.325\textwidth}
		\includegraphics[width=\linewidth]{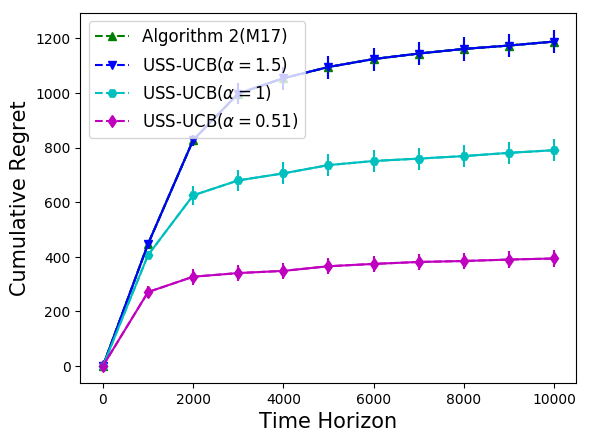}
		\caption{Synthetic BSC Dataset}
		\label{fig:compare_algos_bsc}
	\end{subfigure}
	\begin{subfigure}[b]{.325\textwidth}
		\includegraphics[width=\linewidth]{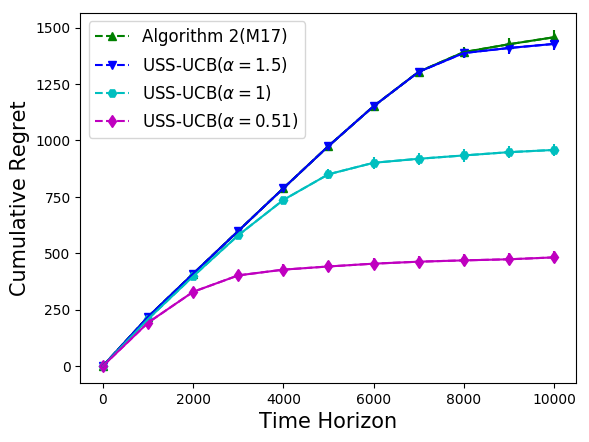}
		\caption{Pima Indian Diabetes}
		\label{fig:compare_algos_diabetes}
	\end{subfigure}
	\begin{subfigure}[b]{.325\textwidth}
		\includegraphics[width=\linewidth]{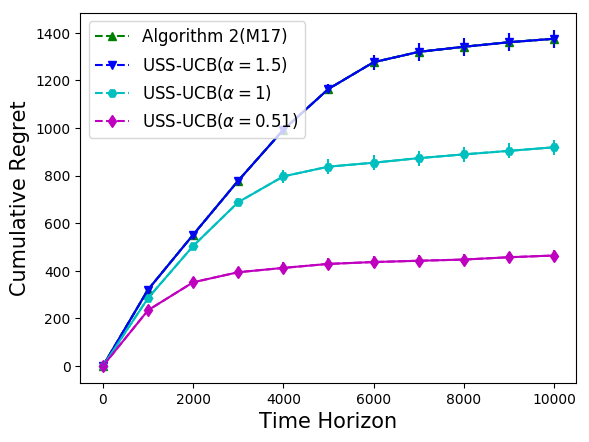}
		\caption{Heart Disease}
		\label{fig:compare_algos_heart}
	\end{subfigure}
	\caption{\footnotesize Comparison between Heuristic Algorithm 2 proposed in \cite{AISTATS17_hanawal2017unsupervised} and \ref{alg:USS_WD} with parameter  $\alpha = \{1.5, 1, 0.51\}$ for Case 1 of the synthetic BSC dataset and real datasets.}
	\label{fig:compare_algs}
\end{figure*}
\begin{figure*}[!ht]
	\centering
	\begin{subfigure}[b]{.325\textwidth}		
		\includegraphics[width=\linewidth]{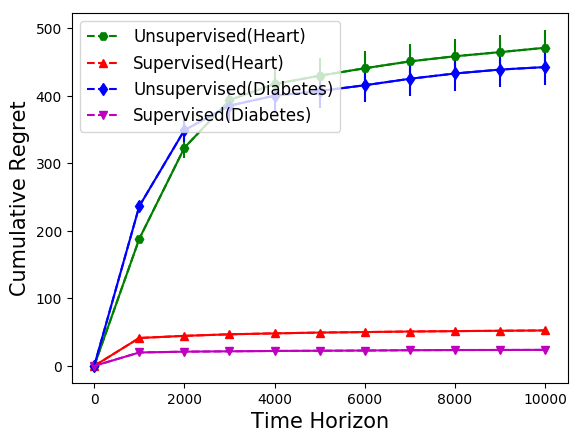}
		\caption{Real Datasets}
		\label{fig:compare_sup_real}
	\end{subfigure}
	\begin{subfigure}[b]{.325\textwidth}		
		\includegraphics[width=\linewidth]{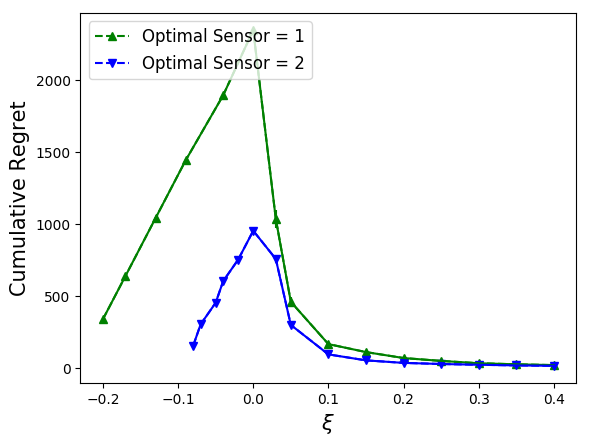}
		\caption{Synthetic BSC Dataset}
		\label{fig:regret_wd_bsc}
	\end{subfigure}
	\begin{subfigure}[b]{.325\textwidth}		
		\includegraphics[width=\linewidth]{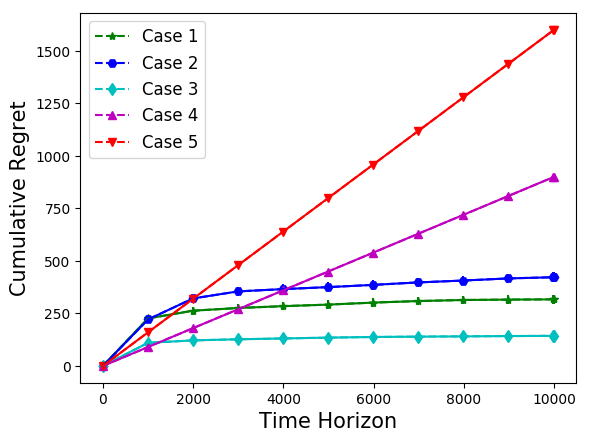}
		\caption{Synthetic BSC Dataset}
		\label{fig:regret_scd_bsc}
	\end{subfigure}
	\caption{\footnotesize Comparison between unsupervised and supervised setting is shown for Case 1 of real datasets \eqref{fig:compare_sup_real}. Cumulative regret v/s WD property for BSC Dataset  using different costs. Right figure: Cumulative regret v/s Time Horizon for synthetic BSC Dataset  when sensor are not ordered by their error rates \eqref{fig:regret_wd_bsc}. Sensor 2 and 3 are interchanged in the sequence while keeping the cost same as given in the Table \ref{table:bsc} for synthetic BSC dataset \eqref{fig:regret_scd_bsc}. Note that, $i^\star = K$ for Case 4 and WD automatically holds but after interchanging last two classifiers, WD does not hold for Case 4.}
	\label{fig:sup_unsup} 
	\vspace{-3mm}
\end{figure*}


\textbf{Synthetic Dataset:} 
We generate synthetic Bernoulli Symmetric Channel (BSC) dataset \citep{AISTATS17_hanawal2017unsupervised} as follows: The input, $Y_t$, is generated from i.i.d. Bernoulli$(0.7)$ random variable. The problem instance used in experiment has three sensors with error rates $\gamma_1 = 0.4, \gamma_2 = 0.1, \gamma_3 = 0.05$. To ensure strong dominance, we impose the condition given in \cref{equ:sd_prop}  during data generation. When sensor $1$ predicts correctly, we introduce error up to 10\%  to the outputs of sensor $2$ and $3$. We use five problem instances by varying the associated cost of each sensor as given in Table \ref{table:bsc}.  

\begin{table}[!h]
	\centering
	\scriptsize
	\setlength\tabcolsep{2.5pt}
	\setlength\extrarowheight{2.5pt}
	\caption{BSC Dataset. WD doesn't hold for Case 5. Optimal classifier's cost is in red bold font.}
	\label{table:bsc}
	\vspace{-2.5mm}
	\begin{tabularx}{0.48\textwidth}{|p{3.1cm}|p{0.87cm}|p{0.87cm}|p{0.87cm}|p{1.54cm}|}
		\hline
		\textbf{Values/Classifiers} & Clf. 1        & Clf. 2        & Clf. 3     & WD Prop.     \\ 
		\hline
		Case 1 Costs& \textcolor{red}{\textbf{0}} & 0.6          & 0.8        &  \multicolumn{1}{c|}{\checkmark}  \\ 
		\hline
		Case 2 Costs& 0         & \textcolor{red}{\textbf{0.15}} & 0.35      & \multicolumn{1}{c|}{\checkmark} \\ 
		\hline
		Case 3 Costs& \textcolor{red}{\textbf{0}} & 0.65           & 0.9      &    \multicolumn{1}{c|}{\checkmark}  \\ 
		\hline
		Case 4 Costs& 0.2        & 0.36         & \textcolor{red}{\textbf{0.4}} &  \multicolumn{1}{c|}{\checkmark}\\ 
		\hline
		Case 5 Costs& 0         & \textcolor{red}{\textbf{0.11}} & 0.22         &  \multicolumn{1}{c|}{\ding{53}} \\
		\hline
	\end{tabularx}
	\vspace{-3mm}
\end{table}

\textbf{Real Datasets:}  
 Both real datasets specify the costs of acquiring individual features. We split these features into three subsets based on their costs and train three linear classifiers on these subsets using logistic regression. For PIMA-Diabetes dataset (\# of samples=768) the first classifier is associated with patient history/profile at the cost of \$6, the 2nd classifier, in addition, utilizes glucose tolerance test (cost \$ 29) and the 3rd classifier uses all attributes including insulin test (cost \$46). For the Heart dataset (\# of samples=297) we associate 1st classifier with the first 7 attributes that include cholesterol readings, blood-sugar, and rest-ECG (cost \$32), the 2nd classifier utilizes, in addition, the thalach, exang and oldpeak attributes that cost \$397 and the 3rd classifier utilizes more extensive tests at a total cost of \$601. We scale costs using a tuning parameter $\lambda$ (since the costs of features are all greater than one) and consider minimizing a combined objective $(\lambda Cost + Error)$ as stated in Section 2. In our setup, high (low)-values for $\lambda$ correspond to low (high)-budget constraint. For example, if we set a fixed budget of \$50, this corresponds to high-budget (small $\lambda$) and low budget (large $\lambda$) for PIMA Diabetes (3rd classifier optimal) and Heart Disease (1st classifier optimal) respectively. For performance evaluation, different values of $\lambda$ are used in five problem instances for both real datasets as given in Table \ref{table:real_dataset}.
\begin{table}[!t]
	\centering
	\scriptsize
	\caption{Real Datasets. WD doesn't hold for Case 5. Optimal classifier's cost is in red bold font.}
	\label{table:real_dataset}
	\setlength\tabcolsep{1.3pt}
	\setlength\extrarowheight{3pt}
	\vspace{-2.5mm}
	\begin{tabularx}{0.48\textwidth}{|p{1.43cm}|p{0.82cm}|p{0.85cm}|p{0.91cm}|p{0.96cm}|p{0.96cm}|p{0.98cm}|p{0.48cm}|}
		\hline
		\multirow{2}{*}{\parbox{2cm}{Values/ \newline Classifiers}} &\multicolumn{3}{|c|}{\bf PIMA-Diabetes}&\multicolumn{3}{|c|}{\bf Heart Disease} &\multicolumn{1}{|c|}{\multirow{4}{*}{\parbox{.48cm}{WD Pro.}}}\\ \cline{2-7} 
		&Clf. 1 &Clf. 2&Clf. 3&Clf. 1 &Clf. 2&Clf. 3&\\ \cline{1-7}
		Error-rate & 0.3125 &0.2331&0.2279&0.29292 &0.20202& 0.14815 &\\ \cline{1-7}
		Cost (in \$) & 4 &29&46&32 &397 &601&\\ 
		\hline
		$\lambda$ in Case  1& \textcolor{red}{\textbf{0.01}}& 0.0106& 0.015&\textcolor{red}{\textbf{0.0001}}& 0.0008& 0.001&\multicolumn{1}{|c|}{\checkmark}\\ 
		\hline
		$\lambda$ in Case  2& 0.01& \textcolor{red}{\textbf{0.004}}& 0.0038&0.0001& \textcolor{red}{\textbf{0.0001}}& 0.00035&\multicolumn{1}{|c|}{\checkmark}\\ 
		\hline
		$\lambda$ in Case  3& \textcolor{red}{\textbf{0.01}}& 0.0113& 0.015&\textcolor{red}{\textbf{0.0001}}& 0.0009& 0.001&\multicolumn{1}{|c|}{\checkmark}\\ 
		\hline
		$\lambda$ in Case  4& 0.0001& 0.0001& \textcolor{red}{\textbf{0.0001}}&0.00001& 0.00004& \textcolor{red}{\textbf{0.0001}}&\multicolumn{1}{|c|}{\checkmark}\\ 
		\hline
		$\lambda$ in Case  5& 0.01& \textcolor{red}{\textbf{0.002}}& 0.0055&0.0042& \textcolor{red}{\textbf{0.0001}}& 0.00027&\multicolumn{1}{|c|}{\ding{53}}\\ 
		\hline
	\end{tabularx}
	\vspace{-3mm}
\end{table}

\textbf{Verifying WD property:}
As we know the error-rate associated with each sensor, we can find an optimal sensor for a given problem instance. Once the optimal sensor is known, WD property is verified by using estimates of disagreement probability after $T$ rounds.

\textbf{Expected Cumulative Regret v/s Time Horizon:} The {\it Expected Cumulative Regret} of \ref{alg:USS_WD} with $\alpha = 0.51$ versus {\it Time Horizon} plots for the Synthetic BSC Dataset and two real datasets are shown in Figure \ref{fig:regret}. These plots verify that any instance that satisfies WD property has sub-linear regret. The online  \ref{alg:USS_WD} selects an instance randomly from the dataset (with replacement) in each round for fixed time horizon. Further, we make a comparison of Algorithm 2 of \cite{AISTATS17_hanawal2017unsupervised} and \ref{alg:USS_WD} for different values of $\alpha$. With same value of $\alpha = 1.5$, Algorithm 2 of \cite{AISTATS17_hanawal2017unsupervised} and \ref{alg:USS_WD} gives same regret whereas \ref{alg:USS_WD} with $\alpha=0.51$ gives best result. as shown in the Figure \ref{fig:compare_algs}.  We verify that if WD holds in any problem instance with the arbitrary ordering of sensors by error rates, then the problem is learnable as shown in Figure \ref{fig:regret_scd_bsc}. We fix the time horizon to 10000 for our experiments. We repeat each experiment 100 times, and average regret with 95\% confidence bound is presented. 

\textbf{Supervised v/s Unsupervised Learning:}  
We compare \ref{alg:USS_WD} against an algorithm where the learner receives feedback. In particular, for each action in each round, in the bandit setting, the learner knows whether or not the corresponding sensor output is correct. We implement the “supervised bandit” setting by replacing Step \ref{alg:USS_WD_est} in \ref{alg:USS_WD} with estimated marginal error rates. We notice that for both high as well as low-cost scenarios, while supervised algorithm does have lower regret, the \ref{alg:USS_WD} cumulative regret is also sublinear as shown in Figure \ref{fig:compare_sup_real}. It is qualitatively interesting because these plots demonstrate that, in typical cases, our unsupervised algorithm learn as good as the supervised setting. 

\textbf{Learnability v/s WD Property:} To verify the relationship between learnability and WD property, we experiment with different problem instances of synthetic BSC dataset that are parameterized by varying costs. We test the hypothesis that set of problem instances satisfying the WD property is a maximal learnable set. We fixed an optimal sensor and vary the costs in such a way that we continuously pass from the situation where WD  holds $(\xi := \min_{j>i^\star} \xi_j$ and $\xi > 0)$ to the case where WD  does not hold $(\xi \le 0$ or $C_j - C_{i^\star} \in (\gamma_{i^\star} - \gamma_j, p_{i^\star j}]$ for any $j>i^\star)$. 
If WD does not hold for any problem instance then \ref{alg:USS_WD} converges to sub-optimal sensor $j$ instead of optimal sensor $i^\star$. In such problem instances, as $C_j - C_{i^\star}$ increase, the cumulative regret \eqref{equ:regret} will also increase due to selection of sub-optimal sensor $j$ by \ref{alg:USS_WD} until WD  does not hold for that problem instance i.e., $\xi > 0$. The difference $C_j - C_{i^\star}$ is lower bounded by $\gamma_{i^\star} - \gamma_j$ in such cases, therefore, $\xi$ cannot be less than $-\Delta_j$. 
We start experiments with the minimum possible value of $\xi$ for which problem instance does satisfy WD property and then increase the value of $\xi$. Figure \ref{fig:regret_wd_bsc} depicts cumulative regret \ref{alg:USS_WD} v/s $\xi$ plots for Synthetic BSC Dataset. It can be seen clearly that there is indeed a transition at $\xi = 0$.

\section{Conclusion}
\vspace{-1.75mm}
We studied the problem of selecting the best sensor in a cascade of sensors where they are ordered according to their prediction accuracies. The best sensor optimally trades-off between sensor costs and their prediction accuracy. The challenge in this setup is that the ground truth is not revealed at any time and hence setup is completely unsupervised. We modeled it as stochastic partial monitoring problem and proposed an algorithm that gives sub-linear regret under the Weak Dominance (WD) property. We showed that our algorithm enjoys regret of order $\tilde{O}(T^{2/3})$ (hiding logarithmic terms) and when the problem instance satisfies the more stringent Strong Dominance property, the regret bound improves to $\tilde{O}(T^{1/2})$. We showed that our algorithm enjoys the same performance under WD property even if the sensor ordering is not necessarily according to the decreasing value of their prediction accuracies.

In the current work, we did not exploit any side information (contexts) available with the tasks. It would be interesting to study the contextual version of this problem where the optimal sensor could be job dependent. 

\section*{Acknowledgment}
\label{sec:uss_Ack}
Arun Verma is partially supported by MHRD Fellowship, Govt. of India. M.K. Hanawal is supported by IIT Bombay IRCC SEED grant (16IRCCSG010) and INSPIRE faculty fellowship (IFA-14/ENG-73) from DST, Govt. of India. 
V. Saligrama acknowledges the support of the NSF through grant 1527618.
AV and MKH would like to thank Prof. N. Hemachandra, IEOR, IIT Bombay for many useful discussions.
This work was done when Csaba Szepesv\'ari was at leave from the University of Alberta.

\bibliographystyle{unsrtnat} 			
\bibliography{ref}

\begin{thebibliography}{21}
\providecommand{\natexlab}[1]{#1}
\providecommand{\url}[1]{\texttt{#1}}
\expandafter\ifx\csname urlstyle\endcsname\relax
  \providecommand{\doi}[1]{doi: #1}\else
  \providecommand{\doi}{doi: \begingroup \urlstyle{rm}\Url}\fi

\bibitem[Hanawal et~al.(2017)Hanawal, Szepesvari, and
  Saligrama]{AISTATS17_hanawal2017unsupervised}
Manjesh Hanawal, Csaba Szepesvari, and Venkatesh Saligrama.
\newblock Unsupervised sequential sensor acquisition.
\newblock In \emph{Artificial Intelligence and Statistics}, pages 803--811,
  2017.

\bibitem[Trapeznikov and Saligrama(2013)]{AISTATS13_trapeznikov2013supervised}
Kirill Trapeznikov and Venkatesh Saligrama.
\newblock Supervised sequential classification under budget constraints.
\newblock In \emph{Artificial Intelligence and Statistics}, pages 581--589,
  2013.

\bibitem[Seldin et~al.(2014)Seldin, Bartlett, Crammer, and
  Abbasi-Yadkori]{ICML14_seldin2014prediction}
Yevgeny Seldin, Peter~L Bartlett, Koby Crammer, and Yasin Abbasi-Yadkori.
\newblock Prediction with limited advice and multiarmed bandits with paid
  observations.
\newblock In \emph{ICML}, pages 280--287, 2014.

\bibitem[Zolghadr et~al.(2013)Zolghadr, Bart{\'o}k, Greiner, Gy{\"o}rgy, and
  Szepesv{\'a}ri]{NIPS13_zolghadr2013online}
Navid Zolghadr, G{\'a}bor Bart{\'o}k, Russell Greiner, Andr{\'a}s Gy{\"o}rgy,
  and Csaba Szepesv{\'a}ri.
\newblock Online learning with costly features and labels.
\newblock In \emph{Advances in Neural Information Processing Systems}, pages
  1241--1249, 2013.

\bibitem[Greiner et~al.(2002)Greiner, Grove, and
  Roth]{AI02_greiner2002learning}
Russell Greiner, Adam~J Grove, and Dan Roth.
\newblock Learning cost-sensitive active classifiers.
\newblock \emph{Artificial Intelligence}, 139\penalty0 (2):\penalty0 137--174,
  2002.

\bibitem[P{\'o}czos et~al.(2009)P{\'o}czos, Abbasi-Yadkori, Szepesv{\'a}ri,
  Greiner, and Sturtevant]{ICML09_poczos2009learning}
Barnab{\'a}s P{\'o}czos, Yasin Abbasi-Yadkori, Csaba Szepesv{\'a}ri, Russell
  Greiner, and Nathan Sturtevant.
\newblock Learning when to stop thinking and do something!
\newblock In \emph{Proceedings of the 26th Annual International Conference on
  Machine Learning}, pages 825--832. ACM, 2009.

\bibitem[Platanios et~al.(2014)Platanios, Blum, and
  Mitchell]{UAI14_platanios2014estimating}
Emmanouil~Antonios Platanios, Avrim Blum, and Tom~M Mitchell.
\newblock Estimating accuracy from unlabeled data.
\newblock In \emph{UAI}, pages 682--691, 2014.

\bibitem[Platanios et~al.(2016)Platanios, Dubey, and
  Mitchell]{ICML16_platanios2016estimating}
Emmanouil~Antonios Platanios, Avinava Dubey, and Tom Mitchell.
\newblock Estimating accuracy from unlabeled data: A bayesian approach.
\newblock In \emph{International Conference on Machine Learning}, pages
  1416--1425, 2016.

\bibitem[Platanios et~al.(2017)Platanios, Poon, Mitchell, and
  Horvitz]{NIPS17_platanios2017estimating}
Emmanouil Platanios, Hoifung Poon, Tom~M Mitchell, and Eric~J Horvitz.
\newblock Estimating accuracy from unlabeled data: A probabilistic logic
  approach.
\newblock In \emph{Advances in Neural Information Processing Systems}, pages
  4361--4370, 2017.

\bibitem[Bonald and Combes(2017)]{NIPS17_bonald2017minimax}
Thomas Bonald and Richard Combes.
\newblock A minimax optimal algorithm for crowdsourcing.
\newblock In \emph{Advances in Neural Information Processing Systems}, pages
  4352--4360, 2017.

\bibitem[Kleindessner and Awasthi(2018)]{ICLM18_kleindessner2018crowdsourcing}
Matth{\"a}us Kleindessner and Pranjal Awasthi.
\newblock Crowdsourcing with arbitrary adversaries.
\newblock In \emph{International Conference on Machine Learning}, pages
  2713--2722, 2018.

\bibitem[Cesa-Bianchi et~al.(2006)Cesa-Bianchi, Lugosi, and
  Stoltz]{MOR06_cesa2006regret}
Nicolo Cesa-Bianchi, G{\'a}bor Lugosi, and Gilles Stoltz.
\newblock Regret minimization under partial monitoring.
\newblock \emph{Mathematics of Operations Research}, 31\penalty0 (3):\penalty0
  562--580, 2006.

\bibitem[Bart{\'o}k and Szepesv{\'a}ri(2012)]{ICML12_bartok2012partial}
G{\'a}bor Bart{\'o}k and Csaba Szepesv{\'a}ri.
\newblock Partial monitoring with side information.
\newblock In \emph{International Conference on Algorithmic Learning Theory},
  pages 305--319. Springer, 2012.

\bibitem[Bart{\'o}k et~al.(2014)Bart{\'o}k, Foster, P{\'a}l, Rakhlin, and
  Szepesv{\'a}ri]{MOR14_bartok2014partial}
G{\'a}bor Bart{\'o}k, Dean~P Foster, D{\'a}vid P{\'a}l, Alexander Rakhlin, and
  Csaba Szepesv{\'a}ri.
\newblock Partial monitoring—classification, regret bounds, and algorithms.
\newblock \emph{Mathematics of Operations Research}, 39\penalty0 (4):\penalty0
  967--997, 2014.

\bibitem[Wu et~al.(2015)Wu, Gy{\"o}rgy, and
  Szepesv{\'a}ri]{NIPS15_wu2015online}
Yifan Wu, Andr{\'a}s Gy{\"o}rgy, and Csaba Szepesv{\'a}ri.
\newblock Online learning with gaussian payoffs and side observations.
\newblock In \emph{Advances in Neural Information Processing Systems}, pages
  1360--1368, 2015.

\bibitem[Auer et~al.(2002)Auer, Cesa-Bianchi, and Fischer]{ML02_auer2002finite}
Peter Auer, Nicolo Cesa-Bianchi, and Paul Fischer.
\newblock Finite-time analysis of the multiarmed bandit problem.
\newblock \emph{Machine learning}, 47\penalty0 (2-3):\penalty0 235--256, 2002.

\bibitem[Garivier and Capp{\'e}(2011)]{COLT11_garivier2011kl}
Aur{\'e}lien Garivier and Olivier Capp{\'e}.
\newblock The kl-ucb algorithm for bounded stochastic bandits and beyond.
\newblock In \emph{Proceedings of the 24th annual Conference On Learning
  Theory}, pages 359--376, 2011.

\bibitem[Kaggle(2016)]{UCI16_pima2016kaggale}
UCI Machine~Learning{,} Kaggle.
\newblock {Pima Indians Diabetes Database}.
\newblock 2016.
\newblock URL
  \url{https://www.kaggle.com/uciml/pima-indians-diabetes-database}.

\bibitem[Detrano(1998)]{HEART98_robert1988va}
Robert Detrano.
\newblock {V.A. Medical Center, Long Beach and Cleveland Clinic Foundation:
  Robert Detrano, MD, Ph.D., Donor: David W. Aha}.
\newblock 1998.
\newblock URL \url{https://archive.ics.uci.edu/ml/datasets/Heart+Disease}.

\bibitem[Dheeru and Karra~Taniskidou(2017)]{UCI17_Dua2017}
Dua Dheeru and Efi Karra~Taniskidou.
\newblock {UCI} machine learning repository, 2017.
\newblock URL \url{http://archive.ics.uci.edu/ml}.

\bibitem[Hoeffding(1963)]{JASA63_hoeffding1963probability}
Wassily Hoeffding.
\newblock Probability inequalities for sums of bounded random variables.
\newblock \emph{Journal of the American statistical association}, 58\penalty0
  (301):\penalty0 13--30, 1963.

\end{thebibliography}

\ifsup
\appendix
\onecolumn
\centerline{\Large \bf Supplementary Material:} \vspace{3mm} \centerline{\Large\bf Online Algorithm for Unsupervised Sensor Selection}
\hrulefill \\

\label{sec:uss_appendix}

\section{Proof of  \cref{lem:B}} 
\label{sec:lemsetbproof}
\label{appendix:proofs}

\SetB*
\begin{proof}
	Let $i^\star$ be an optimal sensor.
	Define 
	\begin{align}
	 	&\mathcal{B}^l :=\big\{i: i \in [2, K];\, \forall j<i \,\,\ni\,\, C_i - C_j \le \mathbb{P}\{ Y^i \ne Y^j \} \big\}  \cup \; \{1\} \label{equ:Bl} \\
	 	&\mathcal{B}^l_{-1} := \mathcal{B}^l\setminus\{1\} \\
		&\mathcal{B}^h := \big\{i: i \in [1, K-1];\, \forall j>i \,\,\ni\,\, C_j - C_i > \mathbb{P}\{ Y^i \ne Y^j \} \big\} \cup \; \{K\}\label{equ:Bh} \\
		&\mathcal{B}^h_{-K} := \mathcal{B}^h\setminus\{K\} \\
		&\mathcal{B} := \mathcal{B}^l \cap \mathcal{B}^h \label{equ:setBlh}
	\end{align}
	Consider following three cases:
	\begin{enumerate}[I.]
		\item $1 < i^\star < K$
		\item $i^\star = 1$
		\item $i^\star = K$
	\end{enumerate}
	\textbf{Case I:} $1 < i^\star < K$ \\As $i^\star$ is an optimal sensor therefore $\forall j>i^\star:\, C_j - C_{i^\star} > \mathbb{P}\{ Y^{i^\star} \ne Y^j \} \Rightarrow C_j - C_{i^\star} \nleq \mathbb{P}\{ Y^{i^\star}\ne Y^j \}\Rightarrow \forall j>i^\star \notin \mathcal{B}^l_{-1}$. If  any sensor $l \in \mathcal{B}^l_{-1}$ then $l\le i^\star$ i.e.,
	\begin{align}
		&\mathcal{B}^l_{-1} = \{l_1, l_2, \ldots, l_m, i^\star\} ~~~~\mbox{ where }1 < l_1 < \cdots < l_m < i^\star	\label{equ:setBl} \\
		&\mathcal{B}^l = \mathcal{B}^l_{-1} \cup \{1\} = \{1, l_1, l_2, \ldots,l_m, i^\star\}\label{equ:setBl1}
	\end{align} 
	Similarly, $\forall j<i^\star:\, C_{i^\star} - C_j \le \mathbb{P}\{ Y^{i^\star} \ne Y^j \} \Rightarrow C_{i^\star} - C_j \ngtr \mathbb{P}\{ Y^{i^\star}\ne Y^j \}\Rightarrow \forall j<i^\star \notin \mathcal{B}^h_{-K}$.  If any sensor $h \in \mathcal{B}^h_{-K}$ then $h \ge i^\star$ i.e.,
	\begin{equation}
		\label{equ:setBh}
		\hspace{-1mm}\mathcal{B}^h_{-K} = \{i^\star, h_1, \ldots, h_n\} \mbox{ where } i^\star < h_1 < \cdots < h_n < K	
	\end{equation} 
	\begin{equation}
		\label{equ:setBhk}
		\mathcal{B}^h = \mathcal{B}^h_{-K} \cup \{K\} = \{i^\star, h_1, h_2, \ldots, h_n, K\}	
	\end{equation} 

	From \eqref{equ:setBlh}, \eqref{equ:setBl1} and \eqref{equ:setBhk}, we get
	\begin{align}
		\mathcal{B} &= \mathcal{B}^l \cap \mathcal{B}^h \nonumber \\
		&= \{1, l_1, l_2, \ldots, i^\star\} \cap \{i^\star, h_1, h_2, \ldots, h_k, K\} \nonumber \\
		\Rightarrow \mathcal{B} &=\{i^\star\} \label{equ:setBi}
	\end{align}
	\textbf{Case II:} $i^\star = 1$ \\ 
	Using \eqref{equ:setBl}, we get $\mathcal{B}^l_{-1} = \phi$, hence $\mathcal{B}^l =\{1\}$. Similarly, using \eqref{equ:setBhk}, we have $\mathcal{B}^h = \{1, h_1, h_2, \ldots, h_n, K\}$ that implies
	\begin{align}
		\mathcal{B} = \{1\} 
		\Rightarrow \mathcal{B} = \{i^\star\} \label{equ:setBii} 
	\end{align}
	\textbf{Case III:} $i^\star = K$\\ 
	Using \eqref{equ:setBh}, we get $\mathcal{B}^h_{-K} = \phi$, hence $\mathcal{B}^h =\{K\}$. Similarly, using \eqref{equ:setBl1}, we have $\mathcal{B}^l = \{1, l_1, l_2, \ldots,l_m, K\}$ that implies
	\begin{align}
		\mathcal{B} = \{K\}  \Rightarrow \mathcal{B} = \{i^\star\} \label{equ:setBiii}
	\end{align}
%
	\eqref{equ:setBi},\eqref{equ:setBii} and \eqref{equ:setBiii} $\Rightarrow \mathcal{B}$ is a singleton set and contains the optimal sensor.
\end{proof}

%

The following definition is convenient for the proof arguments.
\begin{defi}[Action Preference ($\succ_t$)]
	\label{def:preference}
	The sensor $i$ is optimistically preferred over sensor $j$ in round $t$ if:
	\begin{subnumcases}	
	{\hspace{-1cm}i \succ_t j := }
	C_i - C_j \le \hpji +  \Psi_{ji}(t)\; \text{if } j<i \label{def_prefer_l} \\
	C_j - C_i > \hpij +  \Psi_{ij}(t)\; \text{if } j>i \label{def_prefer_h} 
	\end{subnumcases}
\end{defi}

\section{Discussion of Remark \ref{rem:DecisionSet}}

The algorithm can converge to a sub-optimal sensor when we replace the term   $\hpijs + \Psi_{i^\star j}(t)$ in \eqref{set:eBh} by $\hpijs - \Psi_{i^\star j}(t)$. To verify this claim, assume algorithm selects sub-optimal sensor $j$ in around $t$ and $j<i^\star$ then, 
\begin{align*}
C_{i^\star} - C_j & > \hpijs - \Psi_{i^\star j}(t)
\intertext{Since sensor $i^\star$ is not used then there is no update in $\hat{p}_{i^\star j}(t+1)$ but by definition $\Psi_{i^\star j}(t+1) > \Psi_{i^\star j}(t)$ therefore,}
C_{i^\star} - C_j& > \hat{p}_{i^\star j}(t+1) - \Psi_{i^\star j}(t+1).
\end{align*}
Hence sub-optimal sensor will always be preferred over the optimal sensor in the subsequent rounds. This can be avoided by using UC term in \eqref{set:eBh} because,
\begin{align}
&\hat{p}_{i^\star j}(t+1) + \Psi_{i^\star j}(t+1) > \hpijs + \Psi_{i^\star j}(t) \nonumber\\
\intertext{The sub-optimal sensor $j$ will not be preferred after sufficient $n$ rounds,}
\Rightarrow\; & C_{i^\star} - C_j < \hat{p}_{i^\star j}(t+n) + \Psi_{i^\star j}(t+n).
\end{align}
As using LC term can make the decisions stuck to sub-optimal sensor, UC term is used in \eqref{set:eBh}.

\section{Proof of Proposition \ref{prop:meanpulls}}

We first recall the standard Hoeffding's inequality \citep[Theorem 2]{JASA63_hoeffding1963probability} that we use in the proof.
\begin{thm}
	\label{def:che_hoeffiding}
	Let $X_1, \ldots,X_n$  be independent random variables with common range $[0,1]$,  $\mu = \EE{X_i}$, and $\hat{\mu}_n = \frac{1}{n}\sum_{t=1}^n X_t$. Then for all $\epsilon \ge 0$,
	\begin{subequations}
		\begin{align}
		&\Prob{\hat{\mu}_n - \mu \le -\epsilon} \le e^{-2n\epsilon^2} \label{def:ch_hf1} \\
		&\Prob{\hat{\mu}_n - \mu \ge \epsilon} \le e^{-2n\epsilon^2} \label{def:ch_hf2} 
		\end{align}
	\end{subequations}
\end{thm}

We need the following lemmas to prove the Proposition \ref{prop:meanpulls}.
\begin{lem}
	\label{lem:erf_integral}
	\begin{align}
		\operatorname{erf}(x) = \int_0^x \mathrm{e}^{-t^2}dt = \int \mathrm{e}^{-x^2}dx \label{equ:erf_integral}
	\end{align}
\end{lem}

\begin{proof}
	Leibniz's rule for $0< g(x)\le h(x)<\infty$,
	\begin{align*}
		\frac{d}{dx}\int_{g(x)}^{h(x)} &f(x,t) dt = f(x,h(x))\frac{dh(x)}{dx} - f(x,g(x))\frac{dg(x)}{dx} + \int_{g(x)}^{h(x)} \frac{\partial f(x,t)}{\partial x} dt
	\end{align*}
	Leibniz's integral rule without any common variable,
	\begin{align*}
		\frac{d}{dx}\int_{g(x)}^{h(x)} f(t) dt = f(h(x))\frac{dh(x)}{dx} - f(g(x))\frac{dg(x)}{dx}
	\end{align*}
	Using Leibniz's rule in \eqref{equ:erf_integral}, we get
	\begin{align*}
		&\frac{d}{dx}\int_0^x \mathrm{e}^{-t^2}dt = \mathrm{e}^{-x^2}\frac{dx}{dx} - 1\frac{d0}{dx}  = \mathrm{e}^{-x^2} \\
		\Rightarrow\; &{d}\int_0^x \mathrm{e}^{-t^2}dt  = \mathrm{e}^{-x^2} dx
	\end{align*}
	Integrating both side,
	\begin{align*}
		&\int{d}\int_0^x \mathrm{e}^{-t^2}dt = \int \mathrm{e}^{-x^2}dx\\
		\Rightarrow & \int_0^x \mathrm{e}^{-t^2}dt = \int \mathrm{e}^{-x^2}dx \qedhere
	\end{align*}
\end{proof}

\begin{lem}
	\label{lem:integral}
	Let $a,b,c,d \in \mathbb{R}^+$ and $t_c = b\sqrt{at} \pm \sqrt{acd}$. Then
	\begin{align*}
		\int \mathrm{e}^{-t_c^2}dt &= \mp \dfrac{\sqrt{{\pi}cd}\operatorname{erf}(t_c)}{\sqrt{a}b^2}  -\dfrac{\mathrm{e}^{-t_c^2}}{ab^2} + C \\
		&\mbox{where } \operatorname{erf}(x) = \dfrac{2}{\sqrt{\pi}}\int\mathrm{e}^{-x^2}dx ~~~~~ \operatorname{(using \; Lemma\; \ref{lem:erf_integral})}
	\end{align*}
\end{lem}

\begin{proof}
	Let $x = t_c \Rightarrow x = b\sqrt{at} \pm \sqrt{acd}$. Then,
	\begin{align*}
	t = \frac{(x \mp \sqrt{acd})^2}{ab^2}
	\end{align*}
	Now differentiate $x$ w.r.t. $t$,
	\begin{align*}
		\frac{dx}{dt} = \frac{b\sqrt{a}}{2\sqrt{t}} \Rightarrow dt = \frac{2\sqrt{t}}{b\sqrt{a}}dx  = \frac{2(x \mp \sqrt{acd})}{ab^2}dx
	\end{align*}
	By changing the variable from $t$ to $x$ in given integral,
	\begin{align*}
		\int \mathrm{e}^{-t_c^2}dt &= \int \mathrm{e}^{-x^2}\frac{2(x \mp \sqrt{acd})}{ab^2}dx \\
		\Rightarrow \int \mathrm{e}^{-t_c^2}dt&= \frac{2}{ab^2}\int x\mathrm{e}^{-x^2}dx \mp \frac{2\sqrt{cd}}{\sqrt{a}b^2}\int \mathrm{e}^{-x^2}dx \numberthis \label{equ:split_integral}
		\intertext{As $\int \mathrm{e}^{-cx^2}dx = \sqrt{\frac{\pi}{4c}} \operatorname{erf}(\sqrt{c}x) + C$ and  $\int x\mathrm{e}^{-cx^2}dx = -\frac{\mathrm{e}^{-cx^2}}{2c} + C$, then \eqref{equ:split_integral}  with $c=1$ is,}
		&= -\frac{\mathrm{e}^{-cx^2}}{ab^2} \mp \frac{\sqrt{\pi cd}\operatorname{erf}(x)}{\sqrt{a}b^2} + C\\
		\Rightarrow \int \mathrm{e}^{-t_c^2}dt &=  \mp \frac{\sqrt{\pi cd}\operatorname{erf}(t_c)}{\sqrt{a}b^2} -\frac{\mathrm{e}^{-t_c^2}}{ab^2} + C \qedhere
	\end{align*}
\end{proof}

\begin{lem}
	\label{lem:integral_m}
	Let $a,b,c,d \in \mathbb{R}^+$ and $t_0 = cdb^{-2}$. Then
	\begin{align*}
		\int^\infty_{t_0} \mathrm{e}^{-(b\sqrt{at} - \sqrt{acd})^2} dt = \dfrac{\sqrt{{\pi}cd}}{\sqrt{a}b^2} + \dfrac{1}{ab^2} 
	\end{align*}
\end{lem}

\begin{proof}
	Using Lemma \ref{lem:integral} with $t_c = b\sqrt{at} - \sqrt{acd}$. 
	\begin{align*}
	\int_{t_0}^{\infty} \mathrm{e}^{-t_c^2}dt &= \left.\left(\dfrac{\sqrt{{\pi}cd}\operatorname{erf}(t_c)}{\sqrt{a}b^2} -\dfrac{\mathrm{e}^{-t_c^2}}{ab^2}\right)\right\vert_{t_0}^\infty \\ 
	\intertext{Since $t_0 = cdb^{-2},\;t_c = 0$ for $t = t_0,\; \operatorname{erf}(0) = 0$ and $\operatorname{erf}(\infty) = 1$, we get }
	\Rightarrow\int_{t_0}^{\infty} \mathrm{e}^{-t_c^2}dt &= \dfrac{\sqrt{{\pi}cd}}{\sqrt{a}b^2} + \dfrac{1}{ab^2} = \frac{1}{b^2}\left(\sqrt{\dfrac{{\pi}cd}{a}} + \dfrac{1}{a}\right) \qedhere
	\end{align*}
\end{proof}

\begin{lem}
	\label{lem:sum_prob_bound_m}
	Let $b, c, d \in \R^+$, $\{X_t\}_{t\ge1}$ be a sequence of independent random variables, $\hat \mu_t = \frac{1}{t}\sum_{s=1}^t X_s$, and $\mu = \EE{X_t}$ where $X_t \in [0,1] ,\;\forall t$. Then
	\begin{align*}
		\sum_{t=1}^n \Prob{\hat{\mu}_t - \mu \ge b  -  \sqrt{\frac{cd}{t}}} \le   1  + \dfrac{1}{b^2} \left( cd + \sqrt{\frac{{\pi}cd}{2}} + \dfrac{1}{2}  \right)
	\end{align*}
\end{lem}

\begin{proof}
		Assume $t_0=\ceil{cdb^{-2}}$  and $t_0 << n$. We divide sum of the interest into two parts as:
		\begin{align*} 
			\sum_{t=1}^n \Prob{\hat{\mu}_t - \mu \ge b  -  \sqrt{\frac{cd}{t}}} = \sum_{t=1}^{t_0} \Prob{\hat{\mu}_t - \mu \ge b - \sqrt{\frac{cd}{t}}\hspace{-0.035cm}}  +  \sum_{t={t_0}}^n \Prob{\hat{\mu}_t - \mu \ge b - \sqrt{\frac{cd}{t}}\hspace{-0.035cm}} 
		\end{align*}
		As $\Prob{\text{any event}} \le 1$ and for $t>t_0$, $b - \sqrt{\frac{cd}{t}} > 0$. Using Hoeffding's inequality \eqref{def:ch_hf2}, we get
		\begin{align*} 
			\sum_{t=1}^n \Prob{\hat{\mu}_t - \mu \ge b  -  \sqrt{\frac{cd}{t}}} &\le \ceil{t_0} + \sum_{t=\ceil{t_0}}^n \mathrm{e}^{-2\left(b\sqrt{t} -\sqrt{cd}\right)^2} \\ 
			&\le 1 + \dfrac{cd}{b^2} + \int^\infty_{t_0}  \mathrm{e}^{-\left(b\sqrt{2t} -\sqrt{2cd}\right)^2}dt
			\intertext{Now using Lemma \ref{lem:integral_m} with $a= 2$,}
			\sum_{t=1}^n \Prob{\hat{\mu}_t - \mu \ge b  -  \sqrt{\frac{cd}{t}}}&\le 1 + \dfrac{1}{b^2} \left( cd + \sqrt{\frac{{\pi}cd}{2}} + \dfrac{1}{2} \right) 
			\qedhere
	\end{align*}
\end{proof}

\MeanPulls*
\begin{proof}
	Assume $N_T(j)$ be the number of times sensor $j$ is selected till $T$ rounds and $I_t$ be the sensor selected by algorithm at round $t$. Then mean number of pulls for any arm $j$ is: 
	\begin{equation*}
		\EE{N_T(j)} = \EE{\sum\limits_{t=1}^T \one{I_t=j}} = \sum\limits_{t=1}^T\Prob{I_t = j}
	\end{equation*}
	
We prove the proposition by considering the case $j<i^\star$ and $j>i^\star$ separately.

\begin{itemize}
	\item {\bf Case I:} $j < i^\star$\\
	If sensor $j$ is preferred over $i^\star$ at round $t$ then,
	\begin{align*}
	&C_{i^\star} - C_ j > \hpjis  +  \Psi_{ji^\star}(t)&& \left(\mbox{from \ref{def_prefer_h}}\right)
	\intertext{It is easy to verify that $C_{i^\star} - C_ j = \pijs - \xi_j$. 
	By definition, $\hpjis = \hpijs$ and $\Psi_{ji^\star}(t) = \Psi_{i^\star j}(t)$,}
	&\hpijs +  \Psi_{i^\star j}(t) < \pijs - \xi_j && \left(\mbox{from \ref{def_delta}, \ref{def_kappa_l} and \ref{def_xi_l}}\right)
	\end{align*}
	If algorithm selects sensor $j$ in round $t$ then it is preferred over an optimal sensor in that round, i.e.,
	\begin{align*}
	\Prob{I_t = j} &= \Prob{I_t = j, j \succ_t i^\star} \le \Prob{j \succ_t i^\star} = \Prob{\hpijs + \sqrt{\frac{\alpha\log{f(t)}}{\Nijms}} < \pijs - \xi_j} 
	\end{align*}
	As $\Nijms$ is a random variable,  Hoeffding's inequality (\ref{def:ch_hf1}) cannot be directly used here. Let	$\hpijx$ denote the value of $\hpijs$ when $\Nijms =s$. Then, we get  \\
	\begin{align*}
	\Prob{I_t = j} &\le \sum_{s=1}^t \Prob{\hpijx+ \sqrt{\frac{\alpha \log f(t)}{s}} \leq \pijs - \xi_j} \\ 
	&= \sum_{s=1}^t \Prob{\hpijx - \pijs \leq - \left(\xi_j + \sqrt{\frac{\alpha \log f(t)}{s}}\right) } 
	\end{align*}

	Now using Hoeffding's inequality (\ref{def:ch_hf1}),
	\begin{align*}
	\Rightarrow\Prob{I_t = j} &\le \sum_{s=1}^t \mathrm{e}^{-2s\left(\xi_j + \sqrt{\frac{\alpha \log f(t)}{s}}\right)^2} \le \sum_{s=1}^t \left(\mathrm{e}^{-2\xi_j^2s}\;\mathrm{e}^{- 2\alpha \log f(t)}\right) \\ 
	& \le\int_{0}^t \dfrac{\mathrm{e}^{-2\xi_j^2s}}{f(t)^{2\alpha}}ds \le\int_{0}^\infty \dfrac{\mathrm{e}^{-2\xi_j^2s}}{f(t)^{2\alpha}}ds \\
	&\le \left.\left(\frac{\mathrm{e}^{-2\xi_j^2s}}{-2f(t)^{2\alpha}\xi_j^2}\right)\right\vert_{0}^\infty = \dfrac{1}{2\xi_j^2f(t)^{2\alpha}} 
	\end{align*}
	The mean number of time a sub-optimal sensor selected in $T$ rounds is:
	\begin{align*}
		\EE{N_T(j)} = \sum\limits_{t=1}^T\Prob{I_t=j}& \leq \sum\limits_{t=1}^T\dfrac{1}{2\xi_j^2f(t)^{2\alpha}} \leq \dfrac{1}{2\xi_j^2} \sum\limits_{t=1}^\infty\dfrac{1}{f(t)^{2\alpha}} = \dfrac{C}{2\xi_j^2}.
	\end{align*}

	\item {\bf Case II:} $j> i^\star$ \\
	If sensor $j$ is preferred over $i^\star$ at round $t$ then,
	\begin{align*}
	&C_j - C_{i^\star} \le \hpijs  +  \Psi_{i^\star j}(t) && \left(\mbox{\hfill from \ref{def_prefer_l}}\right) \\
	\Rightarrow\;\;&\hpijs +  \Psi_{i^\star j}(t) \ge \pijs + \xi_j && \left(\mbox{from \ref{def_delta}, \ref{def_kappa_h} and \ref{def_xi_h}}\right)
	\end{align*}
	 The mean number of time a sub-optimal sensor selected in $T$ rounds is given by:
	\begin{align*}
		\EE{N_T(j)} &= \sum\limits_{t=1}^T \Prob{I_t =j} = \sum\limits_{t=1}^T \Prob{j \succ_t i^\star, I_t = j} \\
		&=  \sum\limits_{t=1}^T\Prob{\hpijs + \sqrt{\frac{\alpha\log{f(t)}}{\Nijms}} \ge \pijs + \xi_j,\; I_t=j } \\
		&\le \sum\limits_{t=1}^T\Prob{\hpijs +  \sqrt{\frac{\alpha\log{f(T)}}{\Nijms}} \ge \pijs + \xi_j,\; I_t=j} \\
		\intertext{As $\Prob{A\cap B} \le \min\left\{\Prob{A}, \Prob{B}\right\} \Rightarrow \Prob{A\cap B} \le \Prob{A}$ or $\Prob{A\cap B} \le \Prob{B}$, we get}
		&\le \sum\limits_{s=1}^T\Prob{\hpijx +  \sqrt{\frac{\alpha\log{f(T)}}{s}} \ge \pijs + \xi_j}\\
		&= \sum\limits_{s=1}^T\Prob{\hpijx  - \pijs \ge \xi_j - \sqrt{\frac{\alpha\log{f(T)}}{s}}}.
		\intertext{Using Lemma \ref{lem:sum_prob_bound_m}  with $b = \xi_j$, $c=\alpha$, $d=\log{f(T)}$, we get}
		\EE{N_T(j)} &\le 1 + \frac{1}{\xi_j^2}\left(\alpha\log f(T) + \sqrt{\frac{\pi\alpha\log f(T)}{2}} + \frac{1}{2}\right).	\qedhere
	\end{align*}
\end{itemize}
\end{proof}


\section{Proof of Theorem \ref{thm:prob_independent_bound}}

\probIndBound*
\begin{proof}
	Let $N_j(T)$ is the number of times sensor $j$ selected in $T$ rounds. Then expected cumulative regret of \ref{alg:USS_WD} for any instance $\theta \in \TSA$ is:
	\begin{align*}
		\EE{\Regret_T} &= \sum\limits_{j\neq i^\star} \EE{N_j(T)}\Delta_j \\
		&= \sum\limits_{j<i^\star}\EE{N_j(T)}\Delta_j + \sum\limits_{j>i^\star}\EE{N_j(T)}\Delta_j 
	\end{align*}
	Now, using the fact that for $j<i^\star$, $\Delta_j = \xi_j - \kappa_j $ and for $j>i^\star$, $\Delta_j = \xi_j + \kappa_j $, we get
	\begin{align*}
		\EE{\Regret_T} &= \sum\limits_{j<i^\star}\EE{N_j(T)}(\xi_j - \kappa_j) + \sum\limits_{j>i^\star}\EE{N_j(T)}(\xi_j + \kappa_j) \\
		& \leq \sum\limits_{j<i^\star}\EE{N_j(T)}\xi_j + \sum\limits_{j>i^\star}\EE{N_j(T)}(\xi_j + \kappa_j) \\
		\Rightarrow \EE{\Regret_T} & \leq \underbrace{\sum\limits_{j<i^\star}\EE{N_j(T)}\xi_j}_{\EE{\Regret^1_T}} + \underbrace{\sum\limits_{j>i^\star}\EE{N_j(T)}(\xi_j + \beta)}_{\EE{\Regret^{2}_T}} \numberthis \label{equ:regret_def},
	\end{align*} 
	where 
	\begin{equation}
		\forall j,\; \kappa_j \le \beta \;\;\begin{cases}
			=0  \quad \mbox{if }   \theta\in  \TSD \\
			\leq 1  \quad \mbox{if }   \theta\in \TWD \mbox{ and sensors are ordered by their error-rate} \\
			\leq 2 \quad \mbox{if } \theta\in \TWD \mbox{ and sensors are arbitrarily ordered by their error-rates} 
	  	\end{cases}
	\end{equation}
	
	In the following, we set $\xi = \min\limits_{j>i^\star} \xi_j$.  Now we consider the case  $\xi \ge 1$ and $\xi < 1$ separately.
	
	\begin{itemize}
		\item {\bf Case I: }$\xi \ge 1 \Leftrightarrow \forall j > i^\star,\; \xi_j \ge 1$ \\		
		Using Proposition \ref{prop:meanpulls} to upper bound $\EE{\Regret^2_T}$, we get
		\begin{align*}
		\EE{\Regret^2_T}  &\le \sum\limits_{j > i^\star}\Bigg(1 \;+ \frac{1}{\xi_j^2}\Bigg(\alpha\log f(T) \;+ \sqrt{\frac{\pi\alpha\log f(T)}{2}} + \frac{1}{2} \Bigg)\Bigg)(\xi_j + \beta) \\
		&\le K\Bigg((\xi_j + 2)+ \Bigg(\alpha\log f(T) + \sqrt{\frac{\pi\alpha\log f(T)}{2}} + \frac{1}{2} \Bigg)\left(\frac{1}{\xi_j} + \frac{\beta}{\xi_j^2}\right)\Bigg) \\
		\Rightarrow \EE{\Regret^2_T}  &\le K(\xi_j + 2)+ K\Bigg(\alpha\log f(T) + \sqrt{\frac{\pi\alpha\log f(T)}{2}} + \frac{1}{2} \Bigg)\left(\frac{1}{\xi} + \frac{\beta}{\xi^2}\right) \numberthis \label{equ:upper_regret}
		\end{align*} 
		For $\xi \ge 1$, $\left(\frac{1}{\xi} + \frac{\beta}{\xi^2}\right) \le \beta + 1$. Further, by definition $\forall j>i^\star,\;\; \xi_j = C_j - C_{i^\star} - \Prob{\Yis \ne \Yj}$, one can easily verify that $\max\limits_{j > i^\star} \xi_j \le C_K - C_1$ as $\Prob{\Yis \ne \Yj} \ge 0$. Assume $C_K - C_1 \le C_1^K$, then \eqref{equ:upper_regret} can be written as:
		\begin{align*}
			\EE{\Regret^2_T} &\le K(C_1^K + 2)+ (\beta + 1)K\Bigg(\alpha\log f(T) + \sqrt{\frac{\pi\alpha\log f(T)}{2}} + \frac{1}{2} \Bigg)	\\			
			&\le \frac{(2C_1^K + \beta + 5)K}{2} + (\beta + 1)K\Bigg(\alpha\log f(T) + \sqrt{\frac{\pi\alpha\log f(T)}{2}}\Bigg) \numberthis \label{equ:upperRegret1}
		\end{align*}
		
		For any $\xi^\prime$, $\EE{\Regret_T^1}$ can written as:
		\begin{align*}
			\EE{\Regret_T^1}  &= \sum\limits_{\substack{\xi^\prime > \xi_j\\ j < i^\star}} \EE{N_j(T)}\xi_j + \sum\limits_{\substack{\xi^\prime < \xi_j\\ j < i^\star}} \EE{N_j(T)}\xi_j  \\
			& \le \sum\limits_{j < i^\star} \EE{N_j(T)}\xi^\prime + \sum\limits_{\substack{\xi^\prime < \xi_j\\ j < i^\star}} \frac{C}{2\xi_j^2}\xi_j ~~~~~\mbox{(using Proposition \ref{prop:meanpulls})} \\
			\Rightarrow \EE{\Regret_T^1} &\le T\xi^\prime + \frac{CK}{2\xi^\prime} ~~~~~\left(\mbox{since } \sum\limits_{j < i^\star} \EE{N_j(T)} \le T\right) \numberthis \label{equ:lowerRegret}
		\end{align*}
		
		From definition, $\forall j < i^\star,\; \xi_j = \Prob{\Yis \ne \Yj} - (C_{i^\star} - C_j)$. As sensors are ordered by increasing cost, one can verify that $\xi_j < 1$ for $\theta \in \TWD$. With this fact, by combining \eqref{equ:upperRegret1} and \eqref{equ:lowerRegret}, \eqref{equ:regret_def} can be written as:
		\begin{align*}
			\EE{\Regret_T} \le T\xi^\prime + \frac{CK}{2\xi^\prime} + \frac{(2C_1^K + \beta + 5)K}{2} + 3K\Bigg(\alpha\log f(T) + \sqrt{\frac{\pi\alpha\log f(T)}{2}}\Bigg) 
			\intertext{Choose $\xi^\prime = \sqrt{\frac{CK}{T}}$ which maximize the upper bound and we get,}
			\EE{\Regret_T} \le \sqrt{2CKT} + \frac{(2C_1^K + \beta + 5)K}{2} + 3K\Bigg(\alpha\log f(T) + \sqrt{\frac{\pi\alpha\log f(T)}{2}}\Bigg) \numberthis \label{equ:regret_xiM1}
		\end{align*}

		\item {\bf Case II: }$\xi < 1$  \\
		Assume $T \ge T_0$ for $j>i^\star$ such that
		\begin{align}
			1+\frac{1}{\xi_j^2} \left( \alpha\log f(T) + \sqrt{\frac{\alpha \pi \log f(T)}{2}} +\frac{1}{2} \right) \leq \frac{2 \alpha \log f(T)}{\xi_j^2} \label{equ:min_T}
		\end{align}
		For $\alpha =1$ and $T_0  = 56$ \eqref{equ:min_T} holds for all $T \ge T_0$. Let $0<\xi^\prime<\xi$. Then  $\EE{\Regret^2_T}$ can be written as:
		\begin{align*}
			\EE{\Regret_T} & \leq \sum\limits_{\substack{\xi^\prime > \xi_j\\ j < i^\star}} \EE{N_j(T)}\xi_j + \sum\limits_{\substack{\xi^\prime < \xi_j\\ j < i^\star}} \EE{N_j(T)}\xi_j + \sum\limits_{\substack{\xi^\prime > \xi_j\\ j > i^\star}} \EE{N_j(T)}(\xi_j+\beta) + \sum\limits_{\substack{\xi^\prime < \xi_j\\ j > i^\star}} \EE{N_j(T)}(\xi_j+\beta) 
		\end{align*}
		Since $\sum\limits_{\substack{\xi^\prime > \xi_j\\ j < i^\star}} \EE{N_j(T)} \le T$ and for every $j>i^\star, \; \xi_j > \xi^\prime$. Using Proposition \ref{prop:meanpulls} and \eqref{equ:min_T}, we get
		\begin{align*}
			\EE{\Regret_T} &  \le T\xi^\prime + \sum\limits_{\xi^\prime < \xi_j} \frac{C}{2\xi_j^2}\xi_j + \sum\limits_{\xi^\prime < \xi_j} \frac{2 \alpha \log f(T)}{\xi_j^2}(\xi_j+\beta)\\
			& \le  T\xi^\prime + \frac{CK}{2\xi^\prime} + 2\alpha K\log f(T)\left(\frac{1}{\xi^\prime} + \frac{\beta}{{\xi^\prime}^2}\right)
			\intertext{As $C = \lim\limits_{T \rightarrow \infty}\sum\limits_{t=1}^T\frac{1}{t^{2\alpha}}$, one can verify that for $T_0 = 56$ and $\alpha=1$, $C < 2\alpha\log f(T)$ holds. Using this fact,}
			&\le  T\xi^\prime +  4\alpha K\log f(T)\left(\frac{1}{\xi^\prime} + \frac{\beta}{{\xi^\prime}^2}\right) \\
			\EE{\Regret_T} &\le T\xi^\prime +  4\alpha K\log f(T)\left(\frac{1}{\xi^\prime} + \frac{\beta}{{\xi^\prime}^2}\right)	\numberthis \label{equ:regret_xiL1} \\
		\end{align*}		
	\end{itemize}

	We first consider $\TWD$ class of problems. For $\xi^\prime < 1$ and $\beta \le 2$, we have $\left(\frac{1}{\xi^\prime} + \frac{\beta}{{\xi^\prime}^2}\right) \le \frac{\beta + 1}{{\xi^\prime}^2}\le \frac{3}{{\xi^\prime}^2}$. Then
	\begin{align*}
		\EE{\Regret_T} &\le T\xi^\prime + \frac{12\alpha K\log f(T)}{{\xi^\prime}^2}
		\intertext{Choose $\xi^\prime = \left( \frac{24\alpha K\log f(T)}{T}\right)^{1/3}$ which maximize above upper bound and we get,}
		\Rightarrow \EE{\Regret_T} &\le \left(24\alpha K\log f(T)\right)^{1/3}T^{2/3} + \frac{\left(24\alpha K\log f(T)\right)^{1/3}}{2}T^{2/3} \\
		&\le 2\left(3\alpha K\log f(T)\right)^{1/3}T^{2/3} + \left(3\alpha K\log f(T)\right)^{1/3}T^{2/3} \\
		\Rightarrow \EE{\Regret_T} &\le 3\left(3\alpha K\log f(T)\right)^{1/3}T^{2/3} \numberthis \label{equ:regret_xi_wd}
	\end{align*}		
	As $C<2\alpha\log f(T)$ and $K<<T$, it is clear that upper bound in \eqref{equ:regret_xi_wd} is worse than \eqref{equ:regret_xiL1}. Hence it completes our proof for the case when any problem instance belongs to $\TWD$.
	
	Now we consider any problem instance $\theta \in \TSD$. For any $\theta \in \TSD \Rightarrow \forall j \in [K],\; \kappa_j = 0 \Rightarrow \beta = 0$. Hence \eqref{equ:regret_xiL1} can be written as
	\begin{align*}
		\EE{\Regret_T} &\le T\xi^\prime +  \frac{4\alpha K\log f(T)}{\xi^\prime}
		\intertext{Choose $\xi^\prime = \left( \frac{4\alpha K\log f(T)}{T}\right)^{1/2}$ which maximize above upper bound and we get,}
		\Rightarrow \EE{\Regret_T} &\le 2\left(4\alpha KT\log f(T)\right)^{1/2} = 4\left(\alpha KT\log f(T)\right)^{1/2} \numberthis \label{equ:regret_xi_sd}
	\end{align*}
	As $C<2\alpha\log f(T)$ and $K<<T$ then upper bound of expected regret in \eqref{equ:regret_xiL1} is $3\left(\alpha KT\log f(T)\right)^{1/2}$ which is better than \eqref{equ:regret_xi_sd}. It complete proof for second part of Theorem \ref{thm:prob_independent_bound}.
\end{proof}

\section{Proof of Proposition \ref{prop:CostRange1}}
\PropCostRangeHigh*
\begin{proof}
	Assume that $C_j-C_i \geq  \max \{0, \gamma_i-\gamma_j \}$. Since $C_j -C_i \notin \left[\max\{0, \gamma_i-\gamma_j\}, \Prob{Y^i \neq Y^j}\right]$, we get $C_j-C_i > \Prob{Y^j \neq Y^i}$.
	
	The other direction follows by noting that  $ \Prob{Y^j \neq Y^i}\geq  \max \{0, \gamma_i-\gamma_j \}$.
\end{proof}

\section{Proof of Proposition \ref{prop:CostRange2}}
\PropCostRangeLow*
\begin{proof}
	Assume that $C_i-C_j \leq  \max\{0, \gamma_j-\gamma_i\} $. Since $\max\{0, \gamma_j-\gamma_i\} \leq \Prob{Y^i \neq Y^j}$, we get $C_j-C_i \leq \Prob{Y^i \neq Y^j}$.
	
	The condition $C_j-C_i \leq \Prob{Y^i \neq Y^j}$ along with $	C_i -C_j \notin \left (\max\{0, \gamma_j-\gamma_i\}, \Prob{Y^i \neq Y^j}\right ]$ implies the other direction, i.e., $C_i-C_j \leq  \max\{0, \gamma_j-\gamma_i\} $. 
\end{proof}

\section{Proof of Proposition \ref{prop:WD2}}
\PropWDUnorder*
\begin{proof}
	From Proposition \ref{prop:CostRange1} and \ref{prop:CostRange2}, if the optimal sensor satifies
	for $j> i^\star$
	\[C_j -C_{i^\star} \notin \left (\max\{0, \gamma_i-\gamma_j\}, \Prob{Y^{i^\star} \neq Y^j}\right] \]
	and for $j< i^\star$
	\[C_{i^*} -C_j \notin \left (\max\{0, \gamma_j-\gamma_i\}, \Prob{Y^{i^\star}\neq Y^j}\right ],\]
	
	Then, for $j > i^\star, C_j -C_{i^\star}> \gamma_{i^\star} - \gamma_j$	iff 	$C_j -C_{i^\star} >\Prob{Y^{i^\star} \neq Y^j}$ \\
	and for 
	$j < i^\star, C_{i^\star}-C_j\leq  \gamma_{j} - \gamma_{i^\star}$	iff 	$C_j -C_{i^\star} \leq \Prob{Y^{i^\star} \neq Y^j}$. Hence we can use $\Prob{Y^i \neq Y^j}$ as a proxy for $\gamma_i -\gamma_j$ to make decision about the optimal arm.
	
	Now notice that for $j < i^\star$, $C_{i^\star}-C_j \leq \gamma_{j} -\gamma_{i^\star} \le \max\{0, \gamma_{j} -\gamma_{i^\star}\}$ (from Lemma \ref{lem:ErrorRateOrder}). Hence for $j < i^\star$ the condition 	\[C_{i^*} -C_j \notin \left (\max\{0, \gamma_j-\gamma_i\}, \Prob{Y^{i^\star}\neq Y^j}\right ]\] is satisfied. Then, the condition 	\[C_j -C_{i^\star} \notin \left (\max\{0, \gamma_i-\gamma_j\}, \Prob{Y^{i^\star} \neq Y^j}\right] \] for $j > i^\star$ is sufficient for learnability.
\end{proof}
 
\section{Additional experiments for Section \ref{sec:uss_experiments}} 
\textbf{Synthetic Datasets:}
The $d$-dimensional samples are randomly generated. Each sample is represented by $(x_1, \ldots, x_d)$ such that $\forall i,~ x_i$ is drawn from $(-1, 1)$ uniformly at random.  We have generated two such datasets: Synthetic Dataset 1 with $d = 3$ and Synthetic Dataset 2 with $d=5$. Both of these datasets have 10000 samples.

We train five linear classifiers on Synthetic Dataset 1 by varying the hyper-parameters in logistic regression and SVM. We use 80:20 train-test split of the dataset and then compute their error-rate for the whole dataset, i.e., the ratio of the total number of misclassification to total samples. The error-rate and cost of classifiers for the five problem instances are given in Table \ref{table:syn1}. 

\begin{table}[H]
	\centering
	\setlength\tabcolsep{1.2pt}
	\setlength\extrarowheight{2.5pt}
	\caption{Synthetic Dataset 1. For Case 5, WD property does not hold. Optimal classifier's cost is in red bold font.}
	\label{table:syn1}
	\begin{tabularx}{0.483\textwidth}{|p{1.96cm}|p{.93cm}|p{0.93cm}|p{0.93cm}|p{0.93cm}|p{1.09cm}|p{.8cm}|}
		\hline
		\textbf{Values/ \newline Classifiers} &Clf. 1 &Clf. 2&Clf. 3&Clf. 4&Clf. 5&\multicolumn{1}{c|}{\multirow{2}{*}{\parbox{.8cm}{WD Prop.}}}\\ \cline{1-6}
		Error-rate & 0.2877 &0.2448&0.2128&0.1714&0.1371&\\ 
		\hline
		Case  1 Costs & \textcolor{red}{\textbf{0.05}}& 0.20& 0.36& 0.54& 0.75&\multicolumn{1}{c|}{\checkmark}\\ 
		\hline
		Case  2 Costs & 0.02& \textcolor{red}{\textbf{0.045}}& 0.20& 0.29& 0.4&\multicolumn{1}{c|}{\checkmark}\\ 
		\hline
		Case  3 Costs & 0.01& 0.021& 0.032& \textcolor{red}{\textbf{0.043}}&0.25&\multicolumn{1}{c|}{\checkmark}\\ 
		\hline
		Case  4 Costs & 0.01& 0.022& 0.035& 0.08& \textcolor{red}{\textbf{0.1}}&\multicolumn{1}{c|}{\checkmark}\\ 
		\hline
		Case 5 Costs & 0.01& 0.021& \textcolor{red}{\textbf{0.032}}& 0.1& 0.133&\multicolumn{1}{c|}{\ding{53}}\\ 
		\hline	
	\end{tabularx}
\end{table}

Similar to Synthetic Dataset 1, we train four linear classifiers on Synthetic Dataset 2. Their error-rate and associated cost for the five problem instances are given in Table \ref{table:syn2}.
\begin{table}[H]
	\centering
	\setlength\tabcolsep{2.5pt}
	\setlength\extrarowheight{2.5pt}
	\caption{Synthetic Dataset 2. For Case 5, WD property does not hold. Optimal classifier's cost is in red bold font.}
	\label{table:syn2}
	\begin{tabularx}{0.48\textwidth}{|p{2.12cm}|p{1cm}|p{1cm}|p{1cm}|p{1cm}|p{0.96cm}|}
		\hline
		\textbf{Values/ \newline Classifiers} &Clf. 1 &Clf. 2&Clf. 3&Clf. 4&\multicolumn{1}{c|}{\multirow{2}{*}{\parbox{.81cm}{WD Prop.}}}\\ \cline{1-5}
		Error-rate & 0.2340 &0.1977 &0.1673 &0.1418&\\ 
		\hline
		Case  1 Costs & \textcolor{red}{\textbf{0.05}}& 0.2& 0.36& 0.6&\multicolumn{1}{c|}{\checkmark}\\ 
		\hline
		Case  2 Costs & 0.022& \textcolor{red}{\textbf{0.045}}& 0.36& 0.6&\multicolumn{1}{c|}{\checkmark}\\ 
		\hline
		Case  3 Costs & 0.01& 0.021& \textcolor{red}{\textbf{0.032}}& 0.2&\multicolumn{1}{c|}{\checkmark}\\ 
		\hline
		Case  4 Costs & 0.01& 0.021& 0.032& \textcolor{red}{\textbf{0.045}}&\multicolumn{1}{c|}{\checkmark}\\ 
		\hline
		Case 5 Costs & 0.005& 0.011& \textcolor{red}{\textbf{0.018}}& 0.075&\multicolumn{1}{c|}{\ding{53}}\\ 
		\hline
	\end{tabularx}
\end{table}
\vspace{-5.85mm}

\begin{figure*}[!h]
	\centering
	\begin{minipage}[b]{.322\textwidth}
		\includegraphics[width=\linewidth]{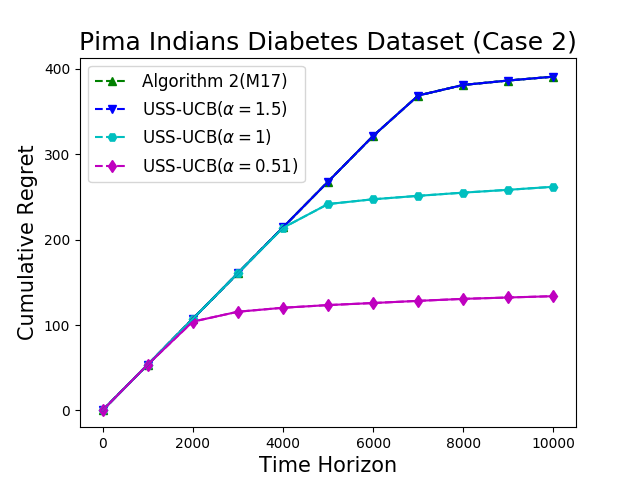}
	\end{minipage}\quad
	\begin{minipage}[b]{.322\textwidth}
		\includegraphics[width=\linewidth]{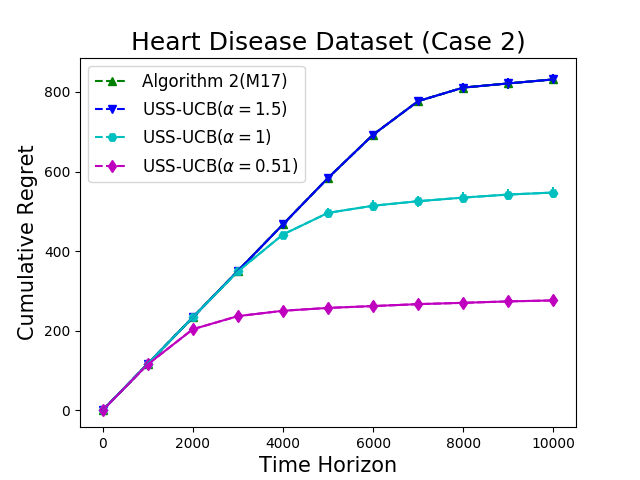}
	\end{minipage}\quad
	\begin{minipage}[b]{.314\textwidth}
		\includegraphics[width=\linewidth]{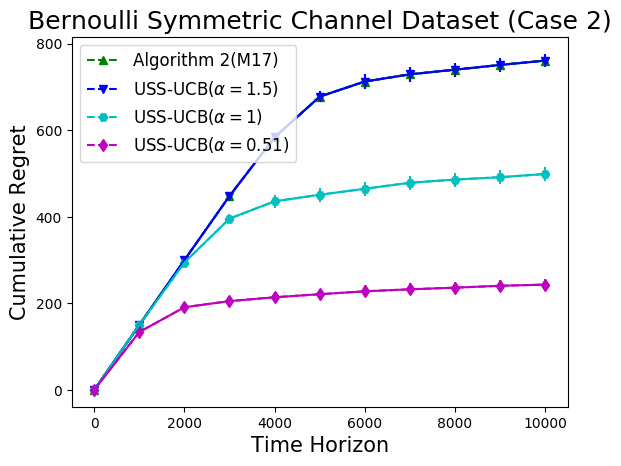}
	\end{minipage}
	\caption{Comparison between Heuristic Algorithm 2 proposed in \cite{AISTATS17_hanawal2017unsupervised} and USS-UCB with parameter  $\alpha = \{1.5, 1, 0.51\}$ for Case 2 of the real datasets and synthetic BSC dataset.}
	\label{fig:compare_algs2}
\end{figure*}

\begin{figure*}[!h]
	\centering
	\begin{minipage}[b]{.322\textwidth}
		\includegraphics[width=\linewidth]{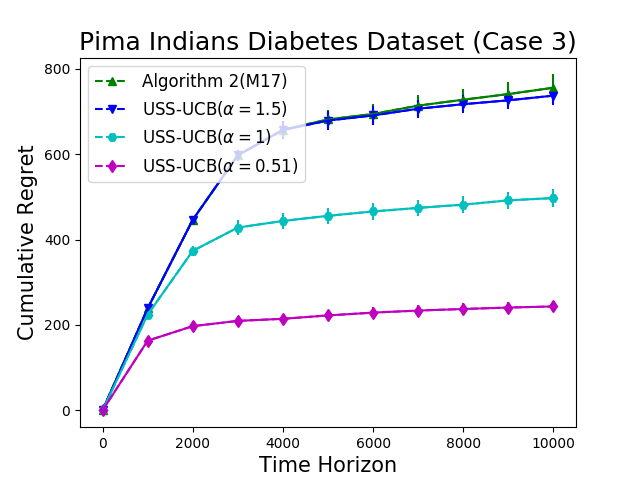}
	\end{minipage}\quad
	\begin{minipage}[b]{.322\textwidth}
		\includegraphics[width=\linewidth]{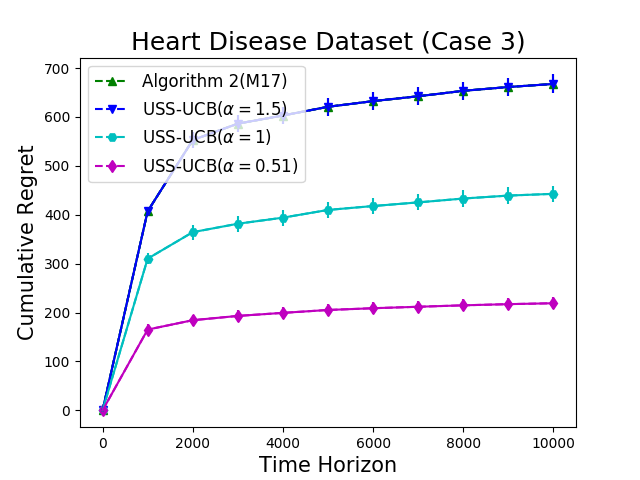}
	\end{minipage}\quad
	\begin{minipage}[b]{.314\textwidth}
		\includegraphics[width=\linewidth]{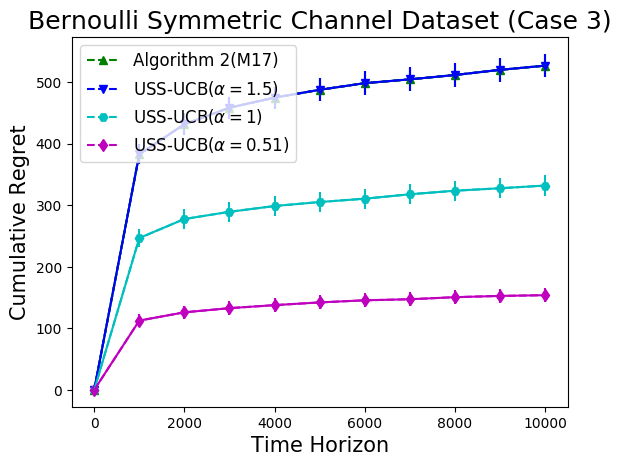}
	\end{minipage}
	\caption{Comparison between Heuristic Algorithm 2 proposed in \cite{AISTATS17_hanawal2017unsupervised} and USS-UCB with parameter  $\alpha = \{1.5, 1, 0.51\}$ for Case 3 of the real datasets and synthetic BSC dataset.}
	\label{fig:compare_algs3}
\end{figure*}

\begin{figure*}[!h]
	\centering
	\begin{minipage}[b]{.32\textwidth}
		\includegraphics[width=\linewidth]{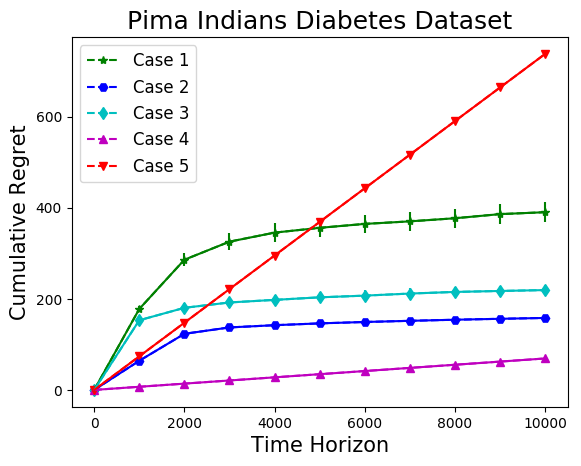}
	\end{minipage}
	\begin{minipage}[b]{.32\textwidth}
		\includegraphics[width=\linewidth]{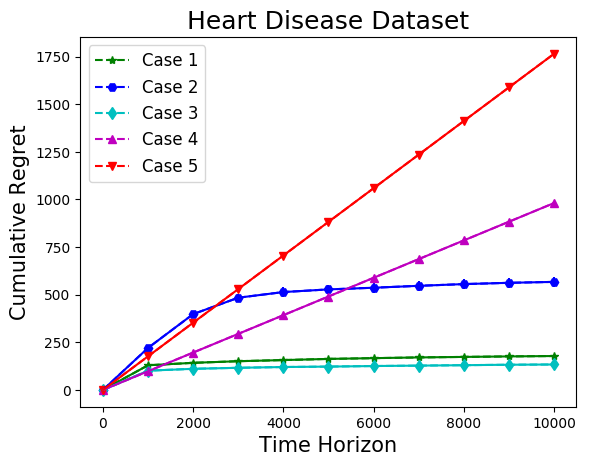}
	\end{minipage}
	\caption{Cumulative regret of \ref{alg:USS_WD}$(\alpha=0.51)$ for different problem instances of the Real Datasets where last two classifier are interchanged in the sequence while keeping the cost same as given in the Table \ref{table:real_dataset}. Note that, $i^\star = K$ for Case 4 and WD automatically holds but after interchanging last two classifiers, WD does not hold for Case 4.}
	\label{fig:scd_real}
\end{figure*}

\begin{figure*}[!h]
	\centering
	\begin{minipage}[b]{.319\textwidth}
		\includegraphics[width=\linewidth]{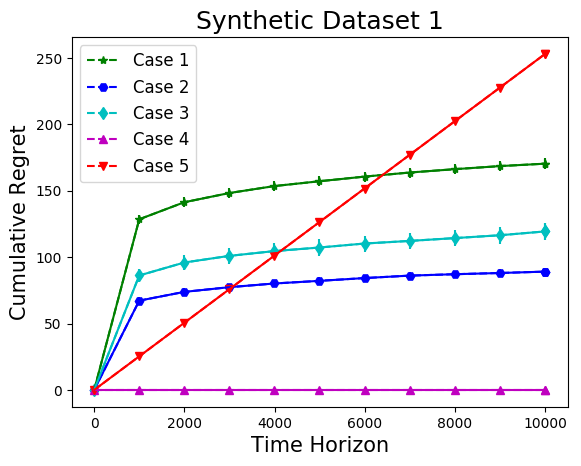}
	\end{minipage}\quad
	\begin{minipage}[b]{.319\textwidth}
		\includegraphics[width=\linewidth]{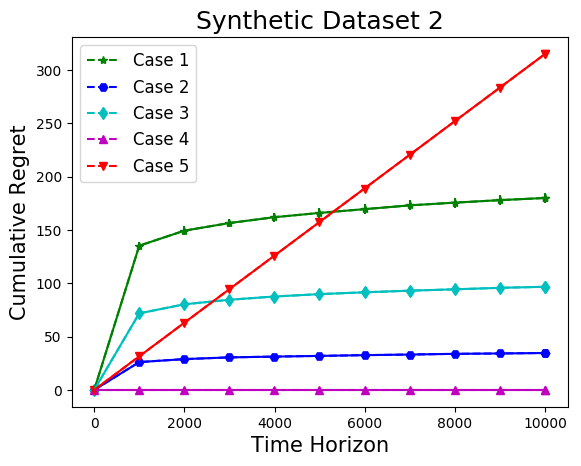}
	\end{minipage}\quad
	\begin{minipage}[b]{.319\textwidth}
		\includegraphics[width=\linewidth]{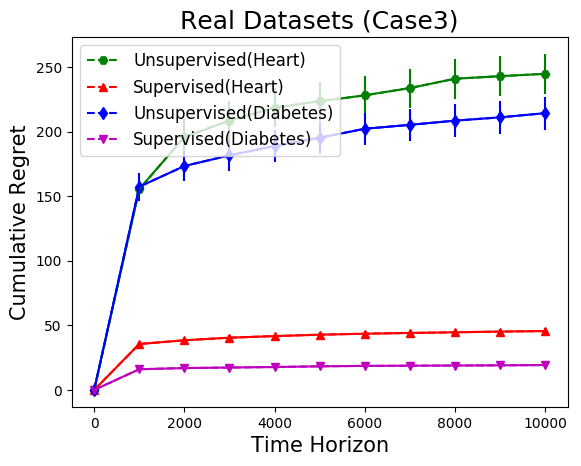}
	\end{minipage}
	\caption{Cumulative regret of \ref{alg:USS_WD}$(\alpha=0.51)$ for different problem instances of the Synthetic Dataset 1 and 2. Rightmost figure: Comparison between unsupervised and supervised setting for Case 3 of real datasets. }
	\label{fig:regret_syn}
\end{figure*}

\begin{figure*}[!h]
	\centering
	\begin{minipage}[b]{.319\textwidth}
		\includegraphics[width=\linewidth]{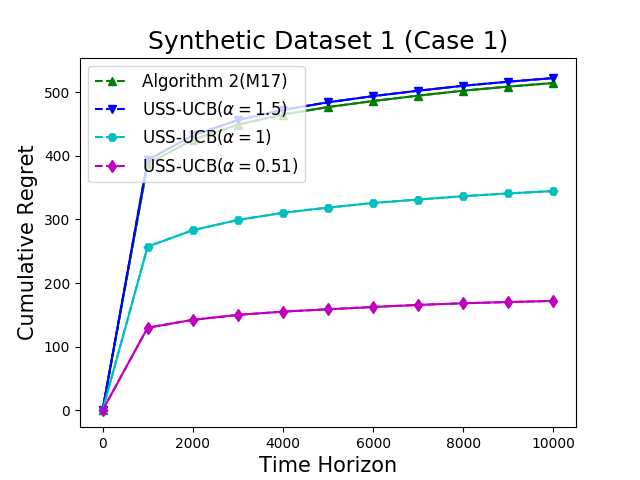}
	\end{minipage}\quad
	\begin{minipage}[b]{.319\textwidth}
		\includegraphics[width=\linewidth]{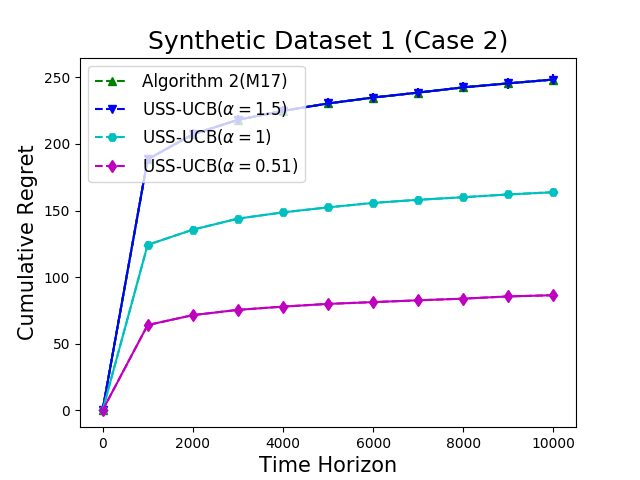}
	\end{minipage}\quad
	\begin{minipage}[b]{.319\textwidth}
		\includegraphics[width=\linewidth]{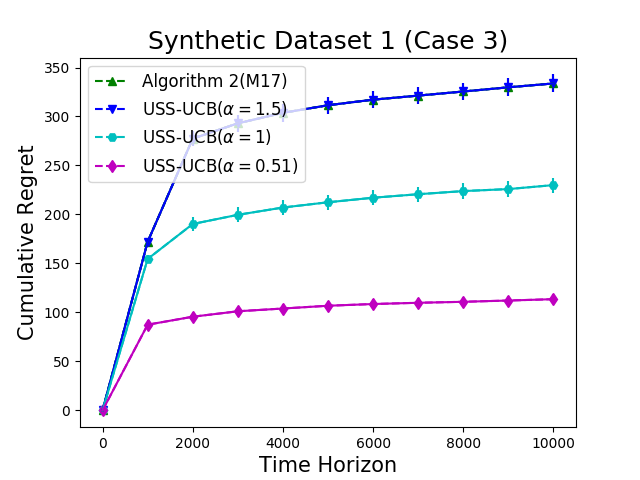}
	\end{minipage}
	\caption{Comparison between Heuristic Algorithm 2 proposed in \cite{AISTATS17_hanawal2017unsupervised} and USS-UCB with parameter  $\alpha = \{1.5, 1, 0.51\}$  for Synthetic Dataset 1.}
	\label{fig:compare_algs_syn1}
\end{figure*}

\begin{figure*}[!h]
	\centering
	\begin{minipage}[b]{.319\textwidth}
		\includegraphics[width=\linewidth]{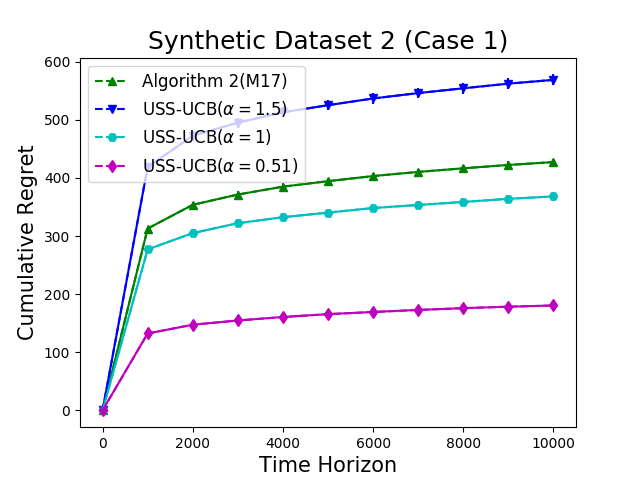}
	\end{minipage}\quad
	\begin{minipage}[b]{.319\textwidth}
		\includegraphics[width=\linewidth]{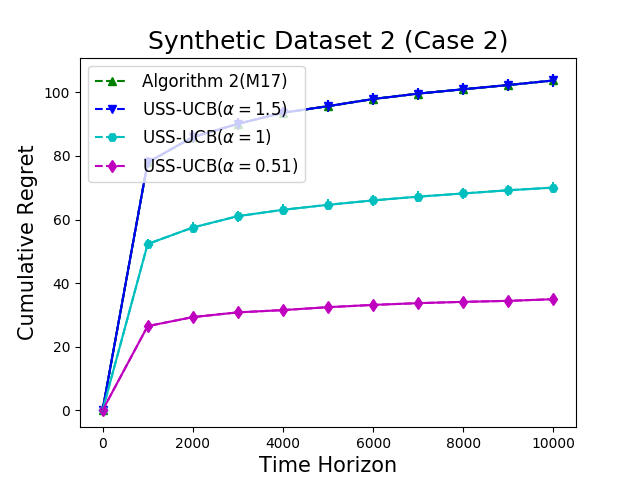}
	\end{minipage}\quad
	\begin{minipage}[b]{.319\textwidth}
		\includegraphics[width=\linewidth]{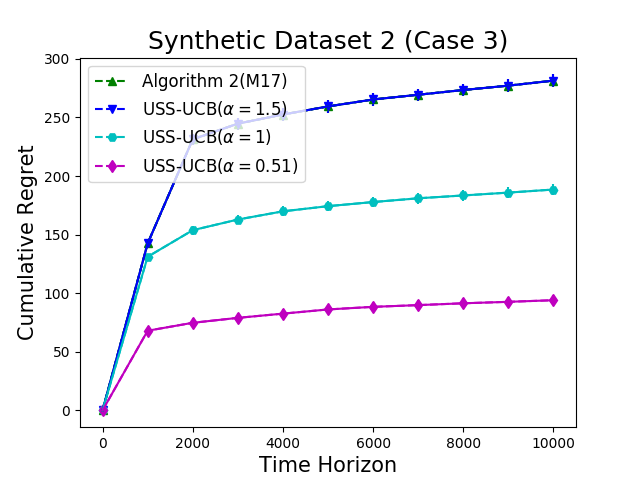}
	\end{minipage}
	\caption{Comparison between Heuristic Algorithm 2 proposed in \cite{AISTATS17_hanawal2017unsupervised} and USS-UCB with parameter  $\alpha = \{1.5, 1, 0.51\}$ for Synthetic Dataset 2.}
	\label{fig:compare_algs_syn2}
\end{figure*}

\begin{figure*}[!h]
	\centering
	\begin{minipage}[b]{.32\textwidth}
		\includegraphics[width=\linewidth]{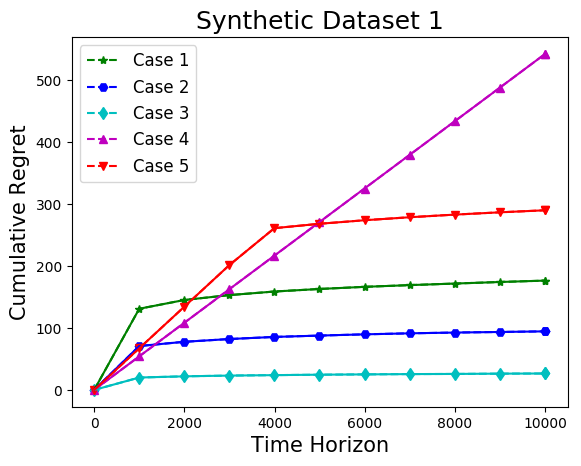}
	\end{minipage}
	\begin{minipage}[b]{.32\textwidth}
		\includegraphics[width=\linewidth]{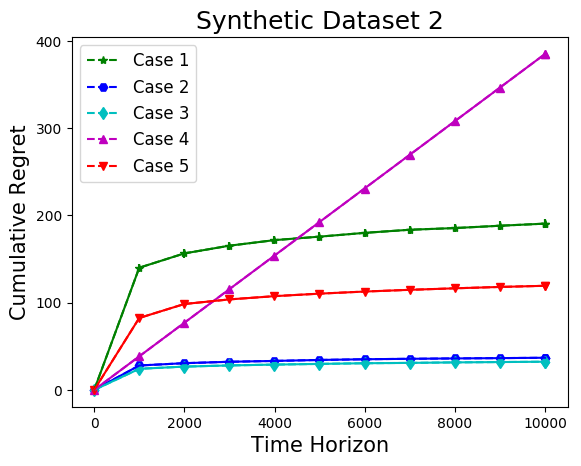}
	\end{minipage}
	\caption{Cumulative regret of \ref{alg:USS_WD}$(\alpha=0.51)$ for different problem instances of the Synthetic Dataset 1 and 2 with last two classifier are interchanged in the sequence while keeping the cost same as given in the Table \ref{table:syn1} and \ref{table:syn2}. Note that, $i^\star = K$ for Case 4 and WD automatically holds but after interchanging last two classifiers, WD does not hold for Case 4 whereas holds for Case 5.}
	\label{fig:scd_syn}
\end{figure*}

\fi

\end{document}